\documentclass[aoas]{imsart}

\RequirePackage{amsthm,amsmath,amsfonts,amssymb}
\RequirePackage[authoryear]{natbib}
\RequirePackage[colorlinks,citecolor=blue,urlcolor=blue]{hyperref}
\RequirePackage{graphicx}
\RequirePackage{natbib}
\usepackage{amsmath}
\usepackage{algpseudocode}
\usepackage{algpseudocode}
\usepackage[linesnumbered,ruled,vlined]{algorithm2e}
\usepackage{natbib}
\usepackage{graphicx}
\usepackage{subfig}
\usepackage{amssymb}
\usepackage[english]{babel}
\usepackage[dvipsnames]{xcolor}
\usepackage{comment}
\usepackage{framed} 
\usepackage{comment}
\usepackage{multirow}
\usepackage{amsthm}
\usepackage{array}
\usepackage{tikz}
\usepackage{bbm}
\usepackage{enumitem}

\numberwithin{equation}{section}

\newcolumntype{P}[1]{>{\centering\arraybackslash}p{#1}}
\newcolumntype{M}[1]{>{\centering\arraybackslash}m{#1}}

\newtheorem{assumption}{Assumption}

\newcommand{\f}{\frac}
\newcommand{\nn}{\nonumber}

\SetKwInput{KwInput}{Input}               
\SetKwInput{KwOutput}{Output}              
\SetKwRepeat{Do}{Repeat}{Until}

\startlocaldefs
\theoremstyle{plain}

\newtheorem{theorem}{Theorem}[section]
\newtheorem{corollary}{Corollary}[section]
\newtheorem{lemma}[theorem]{Lemma}

\newtheorem{proposition}[theorem]{Proposition}

\theoremstyle{definition}

\newtheorem{condition}[theorem]{Condition}

\endlocaldefs

\begin{document}

\begin{frontmatter}
\title{Learning covariate importance for matching in policy-relevant observational research}
\runtitle{Learning covariate importance for matching in policy-relevant observational research}

\begin{aug}

\author[A]{\fnms{Hongzhe} \snm{Zhang}\ead[label=e1]{hongzhe.zhang@pennmedicine.upenn.edu}},
\author[B]{\fnms{Jiasheng} \snm{Shi}\ead[label=e2,mark]{shijiasheng@cuhk.edu.cn}}
\and
\author[A]{\fnms{Jing} \snm{Huang}\ead[label=e3,mark]{jing14@pennmedicine.upenn.edu}}

\address[A]{Department of Biostatistics, Epidemiology and Informatics, Perelman School of Medicine at the University of Pennsylvania, Philadelphia, PA \printead{e1,e3}}
\address[B]{School of Data Science, The Chinese University of Hong Kong, Shenzhen (CUHK-Shenzhen), China \printead{e2}}
\end{aug}

\begin{abstract}
	Matching methods are widely used to reduce confounding effects in observational studies, but conventional approaches often treat all covariates as equally important, which can result in poor performance when covariates differ in their relevance to the study. We propose the Priority-Aware one-to-one Matching Algorithm (PAMA), a novel semi-supervised framework that learns a covariate importance measure from a subset data of units that are paired by experts and uses it to match additional units. It optimizes a weighted quadratic score that reflects the relevance between each covariate and the study, and iteratively updates the covariate importance measure in the score function using unlabeled data. PAMA is model-free, but we have established that the covariate importance measure---the learned weights---is consistent when the oracle matching rule aligns with the design. In addition, we introduce extensions that address imbalanced data, accommodate temporal covariates, and improve robustness to mispaired observations.
	
	
	In simulations, PAMA outperforms standard methods, particularly in high-dimensional settings and under model misspecification. Applied to a real-world study of in-person schooling and COVID-19 transmission, PAMA recovers nearly twice as many expert-designated matches as competing methods using baseline covariates. A self-taught learning extension improves performance in simulations, though its benefit is context-dependent.
	
	To our knowledge, PAMA is the first framework to apply semi-supervised learning to observational matching with covariates of unequal relevance. It offers a scalable and interpretable tool for incorporating expert insight into policy-relevant observational research.
\end{abstract}

\begin{keyword}
\kwd{Matching}
\kwd{Semisupervised}
\kwd{Model-free}
\kwd{Self-taught learning}
\kwd{Covariate importance measure}
\kwd{Temporal covariates}
\end{keyword}

\end{frontmatter}


\section{Introduction}
Matching methods play a vital role in public health policy research by enabling researchers to create comparable groups in observational studies, mimicking the conditions of randomized experiments. In settings where randomization is infeasible or unethical---such as evaluating the effects of school closures, vaccine mandates, or mask policies---matching helps control for confounding by balancing covariates between treated and control units. Accurate and effective matching improves the validity of causal inferences drawn from real-world data---an especially important consideration for informing timely and equitable public health decisions.

A wide range of matching methods have been developed, including propensity score matching, linear discriminant analysis, and more recent machine learning-based approaches such as distance metric learning. These methods differ in how they assess similarity between units---some rely on balancing summary scores, while others attempt to minimize distance across all covariates. While these approaches have proven useful in certain specific circumstances, they often assume all covariates are equally important in determining matching quality, rely on arbitrary prespecified weighting schemes that may not reflect domain knowledge or policy relevance, are not tailored for contemporary observational studies with a vast number of covariates, or fall short in interpretability. These shortcomings, however, can be particularly detrimental in the context of public health applications. For example, when evaluating school reopening policies during a pandemic, not all variables carry equal importance and experts may agree that community transmission rate, school type, and student density should take precedence over socioeconomic indicators or regional location indicators. Existing methods offer limited ability to encode such structured preferences in the matching process. As a result, important contextual priorities may be overlooked, and matching quality may not align with substantive policy considerations.


To address this gap, we propose a novel matching method that prioritizes critical covariates through a learned weighted quadratic score, while progressively accommodating less significant ones. By integrates expert-informed or data-driven variable ordering into the matching algorithm, this approach enables more interpretable, transparent, and policy-relevant comparisons. Furthermore, to account for the sparsely distributed nature of most real-world data applications in high-dimensional space, we focus on refining and adapting our proposed matching algorithm to effectively address one-to-one, also known as bigraph or bipartite graph, matching tasks. By bridging the divide between expert judgment and algorithmic matching, while carefully giving considerations to data characteristics in high-dimensional contexts, our method advances the methodological toolkit for evaluating complex public health interventions.




In addition to resolving the aforementioned inadequacies in existing algorithms, this work also delves into a distinct challenge: the imbalance between labeled and unlabeled observations. This issue, frequently encountered in public health research, arises from the prohibitive cost of expert manual labeling and compels us to further introduce a self-taught learning extension. By exploiting information from both labeled and unlabeled groups in the semi-supervised dataset, we embedded the self-taught learning into the matching algorithm, reducing overfitting and mitigating over-reliance on minority-labeled samples in imbalanced public health datasets.

The remainder of the paper is organized as follows. We introduce the motivating example and its data strcture in section 2 and review the related methods in section 3. The proposed matching algorithm PAMA, its properties, and extensions, are explained in detail in section 4. Simulations are presented in section 5 to evaluate PAMA’s matching performance, while its application to a real-world public health dataset are carried out in section 6. Miscellaneous discussions are left to section 7.

\section{Motivating Examples}
The COVID-19 pandemic created an unprecedented natural experiment in public health policymaking. Across states, cities, school districts, and healthcare systems, a wide range of non-pharmaceutical interventions---such as lockdowns, mask mandates, business restrictions, and school reopening policies---were implemented at different times and under varying local conditions. In the aftermath, researchers and policymakers continue to struggle with evaluating the effectiveness and equity of these interventions: were certain policies beneficial, for whom, and under what contextual conditions?

A core challenge in answering these questions is identifying appropriate comparators for policy adopters. Study units in such evaluations are often large jurisdictions---states, cities, or organizations---that vary across numerous dimensions, including population demographics, healthcare infrastructure, economic indicators, political preferences, comorbidity burdens, and climate patterns. Further adding to the complexity, in one of our motivating examples---a real-world public health study estimating the association between in-person and virtual education and community COVID-19 case incidence following school reopenings during the first year of the pandemic---some covariates within the semi-supervised dataset (e.g., local transmission rate) are time-varying, while others (e.g., population density or hospital capacity) may exhibit extreme heterogeneity or sparsity across regions.

Crucially, another challenge is, the relevance of each covariate varies depending on the specific policy under evaluation. For instance, when assessing mask mandates in schools, factors such as student age distribution, classroom density, and ventilation are likely more significant than median income or regional political leanings. Conversely, when evaluating statewide lockdown policies, economic resilience and employment structure may take precedence. Moreover, many of these covariates are continuous and sparse, making exact matching infeasible and forcing compromises in matching precision. In this context, attempting to match all covariates equally is often impractical and risks diluting the matching quality for the most substantively important variables.

These challenges underscore the need for matching methods that can accommodate variable prioritization---enabling high-relevance covariates to be matched first while flexibly incorporating less critical ones as data permit. A matching framework capable of learning or encoding such prioritization has the potential to yield more meaningful and policy-relevant comparisons, particularly in complex, high-dimensional observational settings such as pandemic response evaluation.

%
%
%

\section{Review of Related Methods}
The challenges outlined in the previous section highlight the need for flexible, data-driven matching frameworks. Existing approaches such as \textit{Distance Metric Learning (DML)} and \textit{Linear Discriminant Analysis (LDA)} provide foundational tools for learning similarity structures from labeled data, but they fall short in addressing key requirements for policy-relevant matching.

DML seeks to learn a transformation of the covariate space that minimizes distances between similar pairs and maximizes distances between dissimilar ones. Let \( x_i, x_j \in \mathbb{R}^p \) be two observations in a $p$-dimensional covariate space. A common formulation of DML, introduced by \citet{xing2002distance}, aims to learn a Mahalanobis distance associated with a symmetric positive semi-definite matrix \( A \in \mathbb{R}^{p \times p} \), such that similar pairs are close and dissimilar pairs are far apart. Let \( \mathcal{S} \) denote a set of similar pairs (e.g., matched units), and \( \mathcal{D} \) denote a set of dissimilar pairs. The learning objective can be formulated as,
\begin{equation*}
\min_{A \succeq 0} \sum_{(x_i, x_j) \in \mathcal{S}} d_A(x_i, x_j) \quad \text{subject to} \quad \sum_{(x_i, x_j) \in \mathcal{D}} d_A(x_i, x_j) \geq 1,
\end{equation*}
where
\begin{equation*}
	 d_A(x_i, x_j) = \sqrt{(x_i - x_j)^T A (x_i - x_j)}.
\end{equation*}
This framework is highly flexible and can accommodate supervision to tailor distances for downstream tasks such as clustering, retrieval, or matching. 

To improve robustness, subsequent methods introduced regularization. \citet{si2006collaborative} added a Frobenius norm penalty on the Mahalanobis distance generating matrix $A$, while \citet{hoi2010semi} proposed a Laplacian regularizer independent of \( \mathcal{S} \) and \( \mathcal{D} \). \citet{globerson2005metric} introduced an alternative probabilistic formulation that minimized the Kullback-Leibler divergence between empirical pairing indicators and a distance-induced conditional distribution,
\begin{equation*}
	 \min_{A \succeq 0} \sum_{i,j} \text{KL}\Big[ \mathbbm{1}(x_j \text{ paired with } x_i) \, \big\| \, \mathbb{P}_A(x_j \mid x_i) \Big],
\end{equation*} 
where $\mathbb{P}_A(x_j \mid x_i) \propto \exp(-d_A(x_i, x_j))$ given $x_i$. This formulation treats metric learning as a probabilistic matching problem, bridging metric learning and soft clustering.


Meanwhile, LDA, originally developed by \citet{fisher1936use} and further advanced by a series of works at the end of last century (\citet{hastie1994flexible}, \citet{hastie1995penalized}, \citet{hastie1995discriminant}, \citet{hastie1996discriminant}, etc.), also learns a transformation of the feature space. In contrast, its goal is to identify a low-dimensional projection that maximizes the separation between groups while minimizing within-group variance. LDA can be interpreted as an approximation of the Bayes classifier, which yield the smallest possible total number of misclassified observations \citep{gareth2023introduction}. Nevertheless, LDA prioritizes directions that maximize class separability, potentially at the expense of incorporating domain-informed variable importance or other contextual insights.

An alternative similar in spirit to LDA, but offering slightly greater flexibility, was introduced by \citet{RCA} and is known as \textit{Relevant Component Analysis (RCA)}. This approach estimates a task-relevant covariance matrix using equivalence constraints derived from groups of similar observations. RCA was subsequently extended by \citet{DCA} through \textit{Discriminative Component Analysis (DCA)}, which learns a projection by optimizing the ratio of between-group to within-group norms of scatter matrices, thereby establishing connections to both DML and LDA in formulation.


While both DML and LDA provide foundational frameworks for learning similarity metrics, they share key limitations in the context of policy evaluation studies. Specifically, they do not allow for the explicit prioritization of variables based on domain knowledge or policy relevance, are not well-suited for contemporary observational studies involving a vast number of covariates, and often fall short in terms of interpretability. Our proposed method builds on the strengths of these approaches while addressing their limitations, enabling more interpretable, structured, and policy-relevant matching in complex observational studies.

\section{The Proposed Priority-Aware one-to-one Matching Algorithm (PAMA)}
\label{Method}

We propose the \textit{Priority-Aware one-to-one Matching Algorithm} (PAMA), a semi-supervised framework that deals with a training set consisting of partially expert-labeled matched pairs and applies this knowledge to match previously unpaired observations within the training set and in incoming task set data streams. PAMA is designed for policy-relevant observational datasets with a large number of covariates, where not all variables are equally relevant for matching. In many applications, experts informally prioritize certain covariates based on domain knowledge, but these priorities are often not explicitly encoded or quantified in traditional matching approaches.

We consider a semi-supervised learning scenario in which a small portion of observation are matched by experts, while a large number remain unpaired due to the prohibitive cost of expert manual matching. We assume that expert pairing, conducted by a panel of health policy experts, is guided by a latent, unobserved function over covariates, informed by domain knowledge and experience. Our goal is to first learn variable importance and then recover this pairing logic from the labeled pairs and use it to perform matching among the unpaired observations in the training set and in an incoming task set.

\vskip 0.5em
The \textit{training set} consists of four subsets of observations:

\begin{itemize}
	\item $\dot{X}_c = \{ \dot{x}_{c,i} \}_{i=1}^{\dot{\ell}}$ and $\dot{X}_t = \{ \dot{x}_{t,i} \}_{i=1}^{\dot{\ell}}$: control and treated observations that have been matched by experts. We organize them such that the $i$-th control observation $\dot{x}_{c,i}$ is matched to the $i$-th treated observation $\dot{x}_{t,i}$, forming the set of matched pairs
	\begin{equation*}
	\mathcal{S}_0 = \left\{ (\dot{x}_{c,i}, \dot{x}_{t,i}) \right\}_{i=1}^{\dot{\ell}},
	\end{equation*}
	where $\dot{\ell}$ is the number of matched pairs.
	
	\item $\ddot{X}_c = \{ \ddot{x}_{c,i} \}_{i=1}^{\ddot{\ell}_c}$ and $\ddot{X}_t = \{ \ddot{x}_{t,i} \}_{i=1}^{\ddot{\ell}_t}$: control and treated observations available in the training set that are not paired by experts.

\end{itemize}

The \textit{task set} contains two sets of unpaired observations that require matching:

\begin{itemize}
	\item $\tilde{X}_c = \{ \tilde{x}_{c,i} \}_{i=1}^{\tilde{\ell}_c}$ and $\tilde{X}_t = \{ \tilde{x}_{t,i} \}_{i=1}^{\tilde{\ell}_t}$: unpaired control and treated observations, unavailable during model training phase.
\end{itemize}

For notational convenience, we define
\begin{equation*}
	 \dot{X} = \begin{pmatrix} \dot{X}_c \\ \dot{X}_t \end{pmatrix}, \quad
	 \ddot{X} = \begin{pmatrix} \ddot{X}_c \\ \ddot{X}_t \end{pmatrix}, \quad
	 \tilde{X} = \begin{pmatrix} \tilde{X}_c \\ \tilde{X}_t \end{pmatrix},\quad {\rm where}\;\; \dot{X}_c = \begin{pmatrix} \dot{x}_{c,1}^T \\ \vdots \\ \dot{x}_{c,\dot{\ell}}^T \end{pmatrix},
\end{equation*}
and similar for $\dot{X}_t, \ddot{X}_c, \ddot{X}_t, \tilde{X}_c$, and $\tilde{X}_t$, where $\dot{X}$ denotes the set of expert-matched observations, $\ddot{X}$ denotes the set of unpaired observations in the training set, and $\tilde{X}$ denotes the task set requiring matching. We do not require the number of unpaired control and treated observations to be equal in either the training or task sets. That is, $\ddot{\ell}_c \ne \ddot{\ell}_t$ and $\tilde{\ell}_c \ne \tilde{\ell}_t$ are allowed. However, we impose the following assumption to ensure that the matching problem is well-posed:

\begin{condition}
	\label{pairing_restrain}
	Under the true data-generating mechanism, each observation in the population has a potential counterpart that constitutes a meaningful match. The training and task sets are random subsamples from this population and may only include subsets of the true pairs. Therefore, within any given dataset, each observation either has a valid match or does not; unmatched units may occur due to random sampling or incomplete data.
\end{condition}

This condition ensures the existence of a meaningful underlying structure to recover, even when expert pairing is partial and the data is imbalanced. It does not imply that every treated observation must be matched, nor that control units outnumber treated ones; instead, it provides the conceptual foundation for learning a pairing logic that generalizes beyond the labeled pairs. In addition, the condition ensures that the distribution of each covariate is consistent across the paired and unpaired datasets for both the control and treated groups. Specifically, for each covariate $k=1,\ldots,p$, the marginal distribution of $x_{c,i,k}$ is assumed to be the same across $\dot{X}_c$, $\ddot{X}_c$, and $\tilde{X}_c$; similarly for $x_{t,i,k}$ across $\dot{X}_t$, $\ddot{X}_t$, and $\tilde{X}_t$.

In practice, it is often reasonable to assume that the expert-based matching rule (pairing logic) can be approximated by a smooth distance or score function. Accordingly, the proposed \textit{Priority-Aware one-to-one Matching Algorithm} (PAMA) learns a quadratic score function to quantify the dissimilarity between control and treatment observations and and performs one-to-one matching based on a tunning threshold.

\subsection{PAMA with Paired Observations}

For simplicity of illustration, we begin with a setting where the number of expert-paired observations is comparable to the number of unpaired observations in the training set, that is, ${\ell} \asymp \min \{  \dot{\ell}_c,\dot{\ell}_t \}$. We denote the quadratic score function used to approximate the expert matching rule as
\begin{equation*}
	S_{\beta}(x_i, x_j) = \beta^T (x_i - x_j)(x_i - x_j)^T \beta,
\end{equation*}
where $x_i$ and $x_j$ are two observations and $\beta \in \mathbb{R}^p$ reflects the importance of each covariate. Heuristically, the learned covariate importance measure $\beta$ should lead to small score differences for expert-matched pairs and, conversely, larger score differences for unmatched pairs.

Leveraging this intuition, we construct an objective function combining a loss term, penalizing large distances among expert-paired observations, and a reward term, encouraging large distances among unmatched observations. To formalize this, we define two adjacency matrices, $W^w$ and $W^b$, which identify matched pairs and unmatched control-treatment pairs among the $2\dot{\ell}$ expert-reviewed observations, respectively,
\begin{gather*}
	W^{w} = (W^w_{ij})_{1\leq i,j \leq 2{\dot{\ell}}}, \quad \text{where} \quad W^w_{ij} = \mathbbm{1}(|i-j| = \dot{\ell}), \\
	W^{b} = (W^b_{ij})_{1\leq i,j \leq 2\dot{\ell}}, \quad \text{where} \quad W^b_{ij} = \mathbbm{1}(|i-j| \neq \dot{\ell}) \cdot \left( \mathbbm{1}(i \leq \dot{\ell} < j) + \mathbbm{1}(i > \dot{\ell} \geq j) \right).
\end{gather*}
With $\dot{x}_i^T$ be the $i-$th row of $\dot{X}$, the \textit{loss} and \textit{reward} functions are then defined as,
\begin{align*} 
	Loss_w(\beta) &=  \frac{1}{{2\dot{\ell}}} \sum_{i,j=1}^{2\dot{\ell}} S_{\beta}(\dot{x}_i, \dot{x}_j) W^w_{ij}= \frac{1}{\dot{\ell}} (\dot{X}\beta)^T (I_{2\dot{\ell}} - W^w) (\dot{X}\beta)\\
	&\triangleq \frac{1}{\dot{\ell}} (\dot{X}\beta)^T L_w (\dot{X}\beta), \\
	Reward_b(\beta) &= \frac{1}{2{\dot{\ell}}^2} \sum_{i,j=1}^{2\dot{\ell}} S_{\beta}(\dot{x}_i, \dot{x}_j) W^b_{ij} = \frac{1}{{\dot{\ell}}^2} (\dot{X}\beta)^T \left( (\dot{\ell}-1) I_{2\dot{\ell}} - W^b \right) (\dot{X}\beta) \\
	&\triangleq \frac{1}{\dot{\ell}^2} (\dot{X}\beta)^T L_b (\dot{X}\beta), 
\end{align*}
where $L_w$ and $L_b$ are commonly referred to as Laplacian matrices in graph theory.
To balance loss minimization and reward maximization, we define the following objective,
\begin{equation*}
	g(\beta) = \frac{Reward_b(\beta)}{Loss_w(\beta) + Penal(\beta)},
\end{equation*}
where $Penal(\beta) = \lambda \|\beta\|_2^2$ is an $\ell_2$ regularization term to prevent overfitting, with $\lambda \geq 0$ denoting a tuning parameter. The estimator $\hat{\beta}$ for the covariate importance measure is obtained by solving
\begin{equation}
	\hat{\beta} = \arg\max_{\beta \in \mathcal{M}_{p\times1}} g(\beta) \quad \text{subject to} \quad \|\beta\|_2 = 1, \label{second opt}
\end{equation}
where $\mathcal{M}_{a\times b}$ is the set of matrices with size $a\times b$. Readily, we have the following proposition,
\begin{proposition}
	\label{prop1}
	Let $\Sigma_b = \dot{X}^T L_b \dot{X}$ and $\Sigma_w = \dot{X}^T L_w \dot{X} / \dot{\ell} + \lambda I_p$. Then, $\hat{\beta}$ is the eigenvector corresponding to the largest eigenvalue of $\Sigma_w^{-1} \Sigma_b$.
\end{proposition}

\subsection{Incorporating Unpaired Observations} The optimization problem \eqref{second opt} initially uses only expert-matched pairs. While PAMA can still operate when the number of expert matched pairs is relatively small, i.e., $\dot{\ell} \ll \min \{ \ddot{\ell}_c, \ddot{\ell}_t \}$, relying solely on a small number of expert-matched pairs may yield an unstable or biased estimate of $\beta$. This challenge becomes more acute when the empirical distribution of covariates differs notably across $\dot{X}$, $\ddot{X}$ and $\tilde{X}$. In such a scenario, we would want PAMA to utilize both $\dot{X}$ and $\ddot{X}$ so that its training data encompasses a broader range of distributions.

\vskip 0.5em

We propose a self-taught learning procedure to leverage the unpaired observations $\ddot{X}$ in the training set by iteratively imputing pairing among them. At the $(k+1)$-th iteration,
\begin{enumerate}
\item Using the current estimate $\hat{\beta}^{(k)}$, search among unpaired observations $\ddot{X}_c$ and $\ddot{X}_t$ for the pair with the smallest score difference.
\item Add the best-matched pair into $\dot{X}$ and remove them from $\ddot{X}$.
\item Update $\hat{\beta}^{(k+1)}$ by re-solving \eqref{second opt}.
\end{enumerate}
More generally, one can repeat this process for a fixed number $\tau_1$ of additions per iteration, or until a convergence criterion is met. While within PAMA,

\begin{enumerate}[label=$\bullet$]
	\item \textbf{What if the smallest score difference is large among unpaired observations?} Except for the extreme case where none of the observations in $\ddot{X}$ should be paired, we can adjust the smallest score differences among unpaired observations to be comparable to the score differences among matched pairs in $\dot{X}$ by introducing a reasonable pairing threshold and a tunable matching critical value $\epsilon$ (details provided later in Algorithm \ref{MCEM1}).
	\item \textbf{How to choose $\bf \tau_1$?}  
	$\tau_1$ is a tuning parameter that trades off between the stability of each update and the speed of convergence. It can be set as a small proportion of the unpaired pool size or determined adaptively based on improvements in $g(\beta)$.
\end{enumerate}

%

The proposed PAMA is presented in detail in Algorithm \ref{MCEM1}.

\begin{algorithm}[htbp!]
	\small
	\caption{Priority-Aware one-to-one Matching Algorithm (PAMA)}
	\label{MCEM1}
	\SetAlgoLined
	
	\KwIn{Training set $\left( \dot{X}_c^T, \dot{X}_t^T, \ddot{X}_c^T, \ddot{X}_t^T \right)^T$}
	
	\textbf{Initialize:}  
	Set iteration counter $k=0$.  
	Estimate initial $\hat{\beta}^{(0)}$ by solving (\ref{second opt}) using the expert-paired observations $\dot{X} = (\dot{X}_c^T, \dot{X}_t^T)^T$.  
	Define $\dot{\ell}^{(0)} = \dot{\ell}$, $\dot{X}_c^{(0)} = \dot{X}_c$, $\dot{X}_t^{(0)} = \dot{X}_t$, $\ddot{X}_c^{(0)} = \ddot{X}_c$, $\ddot{X}_t^{(0)} = \ddot{X}_t$.  
	Set a small threshold $\Delta_0 > 0$ and a tuning integer $\tau_1 > 0$ (number of pairs added per iteration).
	\vskip 0.5em
	
	\While{$\|\hat{\beta}^{(k+1)} - \hat{\beta}^{(k)}\|_{\infty} > \Delta_0$}{
		\For{$s = 1$ to $\tau_1$}{
			\textbf{Find} the control-treatment pair $(\ddot{x}_{c,i_s}, \ddot{x}_{t,j_s})$ from $\ddot{X}_c^{(k)}$ and $\ddot{X}_t^{(k)}$ minimizing $S_{\hat{\beta}^{(k)}}(\cdot,\cdot)$:
			\[
			(\ddot{x}_{c,i_s}, \ddot{x}_{t,j_s}) = \arg\min_{\ddot{x}_{c,i} \in \ddot{X}_c^{(k)},\, \ddot{x}_{t,j} \in \ddot{X}_t^{(k)}} S_{\hat{\beta}^{(k)}}(\ddot{x}_{c,i}, \ddot{x}_{t,j}).
			\]
			\textbf{Update} the unpaired sets:  
			$\ddot{X}_c^{(k)} \leftarrow \ddot{X}_c^{(k)} \setminus \{ \ddot{x}_{c,i_s} \}$,  
			$\ddot{X}_t^{(k)} \leftarrow \ddot{X}_t^{(k)} \setminus \{ \ddot{x}_{t,j_s} \}$.
		}
		\textbf{Augment} the paired set:  
		\[
		\dot{X}_c^{(k+1)} = \Big( \left( \dot{X}_c^{(k)} \right)^T, \ddot{x}_{c,i_1}, \dots, \ddot{x}_{c,i_{\tau_1}} \Big)^T,
		\quad
		\dot{X}_t^{(k+1)} = \Big(  \left( \dot{X}_t^{(k)} \right)^T, \ddot{x}_{t,j_1}, \dots, \ddot{x}_{t,j_{\tau_1}} \Big)^T,
		\]
		and define $\dot{X}^{(k+1)} = \Big( \left( \dot{X}_c^{(k+1)} \right)^T, \left( \dot{X}_t^{(k+1)} \right)^T \Big)^T$.
		
		\textbf{Update} $\hat{\beta}^{(k+1)}$ by re-solving (\ref{second opt}) with $\dot{X}^{(k+1)}$.  
		Increment $k \leftarrow k+1$.
	}
	\textbf{Set} final estimate $\hat{\beta} = \hat{\beta}^{(k)}$.
	
	\BlankLine
	\KwIn{Task set $\left( \tilde{X}_c^T, \tilde{X}_t^T \right)^T$}
	
	\textbf{Initialize:}  
	Set $\tilde{X}_c^{(0)} = \tilde{X}_c$, $\tilde{X}_t^{(0)} = \tilde{X}_t$, $\mathcal{M}_{out} = \emptyset$.  
	Specify a matching threshold $\epsilon > 0$.
	
	\For{$s = 1,2,\ldots$}{
		\textbf{Find} the control-treatment pair $(\tilde{x}_{c,i_s}, \tilde{x}_{t,j_s})$ minimizing $S_{\hat{\beta}}(\cdot,\cdot)$:
		\[
		(\tilde{x}_{c,i_s}, \tilde{x}_{t,j_s}) = \arg\min_{\tilde{x}_{c,i} \in \tilde{X}_c^{(s-1)},\, \tilde{x}_{t,j} \in \tilde{X}_t^{(s-1)}} S_{\hat{\beta}}(\tilde{x}_{c,i}, \tilde{x}_{t,j}).
		\]
		\eIf{($s \leq \min\{\tilde{\ell}_c, \tilde{\ell}_t\}$) \textbf{and} ($S_{\hat{\beta}}(\tilde{x}_{c,i_s}, \tilde{x}_{t,j_s}) \leq \epsilon$)}{
			\textbf{Update} $\mathcal{M}_{out} = \mathcal{M}_{out} \cup \{ (\tilde{x}_{c,i_s}, \tilde{x}_{t,j_s}) \}$.\\
			\textbf{Remove} $\tilde{x}_{c,i_s}$ from $\tilde{X}_c^{(s-1)}$ to form $\tilde{X}_c^{(s)}$, and similarly, \textbf{Remove} $\tilde{x}_{t,j_s}$ from $\tilde{X}_t^{(s-1)}$ to form $\tilde{X}_t^{(s)}$.
		}{
			\textbf{Break} the loop.
		}
	}
	
	\KwOut{Final variable importance vector $\hat{\beta}$ and matched pairs $\mathcal{M}_{out}$ from the task set.}
	
\end{algorithm}

\subsection{Fortifying the Robustness of PAMA}

The iterative structure and the inherently exclusive nature of a one-to-one matching algorithm like PAMA make it vulnerable to error accumulation if incorrectly paired observations are included in early iterations. Even the best-scoring pairs selected at each step may not be an underlying true match, and these errors can degrade the quality of the learned $\beta$ over time. This issue is particularly pronounced when the initial estimate $\hat{\beta}^{(0)}$ is suboptimal or when the data are highly noisy.

To enhance robustness, we propose adding an \textit{exclusion step} alongside the inclusion step, inspired by forward-backward selection methods commonly used in variable selection procedures. This additional step increases variability in the iterative pairing process, preventing overfitting and enhancing stability. Specifically, in the $k$-th iteration,
\begin{enumerate}
	\item \textbf{Inclusion Step:} Identify the top $\tau_1$ candidate pairs with the smallest score differences,
      \begin{equation*}
	 	( \ddot{x}_{c,i_1}, \ddot{x}_{t,j_1} ), ( \ddot{x}_{c,i_2}, \ddot{x}_{t,j_2} ), \ldots, ( \ddot{x}_{c,i_{\tau_1}}, \ddot{x}_{t,j_{\tau_1}} ).
      \end{equation*}
	\item \textbf{Leave-One-Out Re-Estimation:} For each candidate pair, leave it out and re-estimate $\beta$ using the remaining pairs,
	\begin{equation*}
	\hat{\beta}^{(k,s)} = \arg\max_{ \Vert \beta \Vert_2=1, \beta \in \mathcal{M}_{p\times1}} \frac{\sum_{s' \neq s} S_{\beta}(\ddot{x}_{c,i_{s'}}, \ddot{x}_{t,j_{s'}})}{\lambda \|\beta\|_2^2 + \sum_{s_1 \neq s_2, s_1 \& s_2 \neq s} S_{\beta}(\ddot{x}_{c,i_{s_1}}, \ddot{x}_{t,j_{s_2}})}.
	\end{equation*}	
	\item \textbf{Pairing Adjustment:} Using $\hat{\beta}^{(k,s)}$, find the best new match for the left-out observation,
\begin{align*}
	\ddot{x}_{t,j'_s} &= \arg\min_{\dot{x}_{t,u} \in \ddot{X}_t^{(k)}} S_{\hat{\beta}^{(k,s)}}(\ddot{x}_{c,i_s}, \ddot{x}_{t,u}), \\
	\ddot{x}_{c,i'_s} &= \arg\min_{\ddot{x}_{c,u} \in \ddot{X}_c^{(k)}} S_{\hat{\beta}^{(k,s)}}(\ddot{x}_{c,u}, \ddot{x}_{t,j_s}).
\end{align*}
	\item \textbf{Exclusion Step:} Exclude the original pair $(\ddot{x}_{c,i_s}, \ddot{x}_{t,j_s})$ if either best match differs, i.e., if $i'_s \neq i_s$ or $j'_s \neq j_s$. 
	\vskip 0.5em
	\item \textbf{Stop Criterion:} The iterative procedure terminates either when the critical convergence threshold $\Delta_0$ is reached, or if all $\tau_1$ candidate pairs are excluded during an iteration, indicating insufficient pairing quality.
\end{enumerate}

The exclusion step serves as a safeguard against including unreliable pairs, thereby improving the overall robustness of PAMA. However, it should be used judiciously to prevent excessive computational burden.

\subsection{Incorporating Time-Dependent Features and Matching Unequal Size Matrices}

In many observational studies, study units are observed over different durations, resulting in time-varying covariates with irregular observation periods. In this subsection, we describe how to extend PAMA to manage such time-dependent data.

Suppose an observation $\dot{x}_{c,i}$ from the control group $\dot{X}_c$ is measured over $k_{c,i}$ time points. It can be represent in a matrix:
\begin{equation*}
\dot{x}_{c,i} = 
\begin{bmatrix}
	\dot{x}_{c,i,1,1} & \cdots & \dot{x}_{c,i,1,p} \\
	\vdots & \ddots & \vdots \\
	\dot{x}_{c,i,k_{c,i},1} & \cdots & \dot{x}_{c,i,k_{c,i},p}
\end{bmatrix},
\end{equation*}
where each row corresponds to a $p$-dimensional covariate vector observed at a specific time point. Since observations may differ in the number and timing of measurements, the loss and reward functions in PAMA must be adapted to compare multivariate time series of unequal length. To address this, we apply \textit{Dynamic Time Warping} (DTW) \citep[e.g.,][]{muller2007information} to measure similarity between the time-dependent features.

To elaborate, we project the covariate matrix onto a one-dimensional feature space determined by the variable importance measure $\beta$, which is yet to be learned,
\begin{equation*}
z_{c,i,r} = \dot{x}_{c,i,r} \cdot \beta, \quad r=1,\dots,k_{c,i},
\end{equation*}
resulting in a sequence $Z_{c,i} = (z_{c,i,1}, \dots, z_{c,i,k_{c,i}})$ for each unit. For arbitrary $(i,j)$, the distance between the two sequences $Z_{c,i}$ and $Z_{t,j}$ (from the control and the treated groups, respectively) can be defined using either,
\begin{align*}
	{\rm DTW}(Z_{c,i}, Z_{t,j}) &= \min_{A \in \mathcal{A}(k_{c,i},k_{t,j})} \langle D_{Z_{c,i},Z_{t,j}}, A \rangle_F, \\
	 \;\; \text{or} \quad {\rm DTW}_{\gamma}(Z_{c,i}, Z_{t,j}) &= -\gamma \log \sum_{A \in \mathcal{A}(k_{c,i},k_{t,j})} \exp\left( -\frac{1}{\gamma} \langle D_{Z_{c,i},Z_{t,j}}, A \rangle_F \right).
\end{align*}
Here, $D_{Z_{c,i},Z_{t,j}} \in \mathcal{M}_{k_{c,i} \times k_{t,j}}$ is the cost matrix with entries $D_{rs} = (z_{c,i,r} - z_{t,j,s})^2$, $1\leq r\leq k_{c,i}$ and $1\leq s\leq k_{t,j}$. The set $\mathcal{A}(k_{c,i},k_{t,j}) \subset \mathcal{M}_{k_{c,i} \times k_{t,j}}$ denotes the allowable alignment matrices. Specifically, for arbitrary $A\in \mathcal{A}(k_{c,i},k_{t,j})$, we require that $A_{r,s}\in \{0,1\}$ and the set $\{ (r,s): A_{rs}=1, 1\leq r\leq k_{c,i}, 1\leq s\leq k_{t,j} \}$ forms a path connecting $(1,1)$ and $(k_{c,i},k_{t,j})$ \citep[e.g.,][]{wang1997alignment}. Additionally, $\langle\cdot , \cdot \rangle_F$ represents the inner product induced by the Frobenius norm.

\vskip 0.5em

Heuristically, by substituting the DTW distance into PAMA, the loss and reward functions become,
\begin{align*}
	Loss_w(\beta) &= \frac{1}{\ell} \sum_{i,j=1}^{2\ell} \min_{A_{ij} \in \mathcal{A}(k_{c,i},k_{t,j})} \langle D_{Z_{c,i}, Z_{t,j}}, A_{ij} \rangle_F \cdot W^w_{ij},\\
	Reward_b(\beta) &= \frac{1}{\ell^2} \sum_{i,j=1}^{2\ell} \min_{A_{ij} \in \mathcal{A}(k_{c,i},k_{t,j})} \langle D_{Z_{c,i}, Z_{t,j}}, A_{ij} \rangle_F \cdot W^b_{ij}.
\end{align*}
Thus, the variable importance measure $\beta$ can still be estimated by maximizing the reward-to-loss-and-penalty ratio, as described previously, while naturally incorporating time-dependent features through DTW alignment.

\subsection{Asymptotic Properties of $\hat{\beta}$}

Notably, the proposed PAMA is \textit{model-free}. It makes no assumptions regarding the existence of a "true" oracle variable importance measure $\beta$, nor does it assume a specific form for the underlying matching rule.

Accordingly, we assess the quality of the learned $\hat{\beta}$ based on certain \textit{relative consistency} rather than oracle recovery. That is, we evaluate how closely $\hat{\beta}$, learned from a limited expert-matched pairs, approximates $\hat{\beta}^*$, which would have been obtained if the full expert pairing structure were available.

\subsubsection{Recovering Experts' Knowledge from the Master Database}

Consider a setting where the expert-labeled dataset $\dot{X}$ is a subsample of a larger \textit{master database} containing the full matching structure. For notational convenience, we denote the master database as
\begin{equation*}
X^T = (X_c^T, X_t^T) = (\dot{X}_c^T, \ddot{X}_c^{'T},  \dot{X}_t^T, \ddot{X}_t^{'T}),
\end{equation*}
where $\dot{X}^T = (\dot{X}_c^T, \dot{X}_t^T)$ refers to the observations included in the expert-labeled training set, and $\ddot{X}^{'T} = (\ddot{X}_c^{'T}, \ddot{X}_t^{'T})$ is a subset of $\ddot{X}$ representing the remaining paired observations. We emphasize that all observations in the master database $X$ are paired, contrast to the training data sets $(\dot{X}^T,\ddot{X}^T)$.

Extending from the previous notion and setting, the matrices $\dot{X}_c, \dot{X}_t \in \mathcal{M}_{\dot{\ell} \times p}$ contain $\dot{\ell}$ control-treatment pairs labeled by experts, forming the set
\begin{equation*}
	 \mathcal{S}_0 = \left\{ (\dot{x}_{c,i}, \dot{x}_{t,i}) \right\}_{i=1}^{\dot{\ell}}.
\end{equation*}
Similarly, $\ddot{X}_c^{'}, \ddot{X}_t^{'} \in \mathcal{M}_{\ddot{\ell} \times p}$ contain $\ddot{\ell}$ additional control-treatment pairs from the master database, which could be accessed by putting aside the prohibitive cost of expert manual labeling, and are assumed to form the set
\[
\mathcal{S}_0' = \left\{ (\ddot{x}_{c,i}, \ddot{x}_{t,i}) \right\}_{i=1}^{\ddot{\ell}}.
\]
The union $\mathcal{S}_1 = \mathcal{S}_0 \cup \mathcal{S}_0'$ represents the full set of matched pairs in the master database. Let $\hat{\beta}^{(0)}$ denote the variable importance measure estimated from $\dot{X}$ according to (\ref{second opt}), and let $\hat{\beta}^*$ denote the counterpart estimated from the full master database $X$. Thus, under the following moment constraint on the covariate vectors, we establish the relative consistency result,

\begin{assumption}\label{assmp}
	For all random observations in the master database $\{ X_{c,k}, X_{t,k}, k=1,\cdots, \dot{\ell}+\ddot{\ell}  \}$, we assume
	\begin{enumerate}[label=(\roman*)]
		\item For some non-degenerate covariance matrix $\Sigma$ with eigenvalues $\lambda_1 \geq \cdots \geq \lambda_p >0$,
		\begin{equation*}
			\mathbb{E}(X_{c,k}) = \mathbb{E}(X_{t,k}) = 0, \quad \text{and} \quad \mathrm{Var}(X_{c,k}) = \mathrm{Var}(X_{t,k}) = \Sigma.
		\end{equation*}
		\item The leading eigengap of $\Sigma$ is bounded away from zero, i.e., $\exists c>0$, s.t., $(\lambda_{1} - \lambda_2)\geq c$.
		\item $(1+\lambda_{1})^6 p=\min\{ o(\dot{\ell}) , o(\ddot{\ell}) \}$, i.e., the spectral norm of $\Sigma$ would not increase too rapid compare to the sample size.
		\item $X_c$ and $X_t$ are row-wise $\sigma$-sub-Gaussian.
	\end{enumerate}
\end{assumption}


\begin{theorem}\label{main_theorem}
	Under assumption \ref{assmp}, as $p, \dot{\ell}, \ddot{\ell} \to \infty$ with $\dot{\ell}/\ddot{\ell} \to C$ for some $C \in (0,1]$, the estimator $\hat{\beta}^{(0)}$ satisfies
	\begin{equation*}
		 \mathrm{dist}(\hat{\beta}^{(0)}, \hat{\beta}^*) \triangleq \left\| \hat{\beta}^{(0)} \hat{\beta}^{(0)T} - \hat{\beta}^* (\hat{\beta}^*)^T \right\|_2 = O\Big(  (1+\lambda_1)^3 \sqrt{\frac{p}{\dot{\ell}}} \Big),
	\end{equation*}
	which measures the difference between the subspaces spanned by the eigenvectors obtained using only the labeled data and those obtained using the whole master dataset.
\end{theorem}

\noindent The proof is deferred to the Appendix.

\section{Simulations}
\label{simulations}

We applied PAMA, along with RCA \citep{RCA}, DCA \citep{DCA}, Euclidean distance matching, and propensity score matching (denoted as PCM), to both synthetic and real-world datasets. Empirical results demonstrate that PAMA consistently achieves higher matching accuracy compared with existing methods. We also investigate the robustness of PAMA under model misspecifications and assess the potential benefits of the self-taught learning extension.

To provide context for the simulation results, we briefly sketch the baseline performance of a uniformly random matching procedure before presenting the results. We emphasize that achieving high matching accuracy in bipartite (one-to-one) matching problems is intrinsically difficult, especially in large candidate pools. Unlike binary classification, where random guessing yields 50\% accuracy, bipartite matching problems are closely associated with derangements in combinatorial mathematics. Under a uniformly random matching procedure, the limiting probability of observing a derangement is $e^{-1}$, the expected number of correct pairs is always one, and the probability of correctly pairing even a few observations decreases exponentially, as illustrated in Table \ref{tab:random_matching}. Consequently, in our simulation evaluation (for the aforementioned matching algorithms), we expect a matching accuracy ranging from 30\% to 80\%, rather than the 90+\% accuracy typically observed in other types of matching tasks.

\begin{table}[h]
	\centering
	\begin{tabular}{ |c|c|c| }
		\hline
		Number of Pairs & Matching Accuracy & Probability of Derangements (No Correct Matches) \\
		\hline
		5  & 0.2   & 0.328 \\
		10 & 0.1   & 0.356 \\
		15 & 0.066 & 0.362 \\
		20 & 0.05  & 0.358 \\
		30 & 0.033 & 0.365 \\
		50 & 0.02  & 0.370 \\
		\hline
	\end{tabular}
	\vskip 0.5em
	\caption{Performance of Random Matching Procedures}
	\label{tab:random_matching}
\end{table}


In the synthetic simulations, we assume that ``expert knowledge'' comprises known matching rules. PAMA is trained on a training dataset, where a small portion of observations are paired by ``experts'', and is evaluated on held-out test observations generated according to the same matching rules. In the simulation, the tuning parameter $\lambda$ is selected via four-fold cross-validation over a grid of nine logarithmically spaced values spanning $[10^{-4},\,10^{5}]$. Performance is measured by \textit{matching accuracy}, defined as the proportion of correctly matched pairs relative to the ground-truth expert matching. In contrast to methods that evaluate covariate balance \citep{austin2014comparison}, we directly assess the algorithms based on matching fidelity.

\subsection{Matching Accuracy Comparison}

\subsubsection{When the matching rule coincides with the modeling}

We first simulate data where the true matching rule is a weighted quadratic score function. Control observations $\{x_i^c\}_{i=1}^\ell$ are generated from $\mathcal{N}(\mu_i^c, 0.25\cdot I_p)$ with $\mu_i^c = (i, \dots, i)^T$. Treated observations $\{x_j^t\}_{j=1}^\ell$ are generated from $\mathcal{N}(\mu_i^t, 0.25\cdot I_p)$ with $\mu_i^t = \mu_i^c + (b, 0, \dots, 0)^T$, where $b=0.5$ controls the bias between treatment and control groups. The rationale behind this data generation setup stems from the fact that any two multivariate normal distributions, $\mathcal{N}(\mu_1, \Sigma_1)$ and $\mathcal{N}(\mu_2, \Sigma_2)$, can be transformed via an affine mapping into $\mathcal{N}((b,0,\dots,0)^T, I_p)$ and $\mathcal{N}(0, I_p)$, making this a general applicable framework \citep{normal}. Building on the generated observations, we designate the first two covariates as "principal confounders" by assigning them large weights $0.9/(1.7+0.05p)$, while the remaining covariates are given small weights $0.05/(1.7+0.05p)$ to serve as noisy variables.
 
\subsubsection{When the matching rule does not coincide with the modeling}

In this alternative scenario, we adopt the same local data-generating process as above; however, instead of using the quadratic score matching rule, we assume that the expert matches pairs via conjunctive thresholding: two observations are matched only if their differences on the principal confounders (i.e., the first and second coordinates of the observation) fall within a specified threshold $c$, which is set to $1.5$ in the simulation. In other words, the corresponding score function can be defined as
\begin{equation*}
	\tilde{S}(x_i, x_j) = \Big[ (x_{i1} - x_{j1})^2+(x_{i2} - x_{j2})^2 \Big]\cdot \Big[ \mathbbm{1}\big(  |x_{i1} - x_{j1}| < c   \big)\cdot  \mathbbm{1}\big(  |x_{i2} - x_{j2}| < c \big) \Big]^{-1}.
\end{equation*}

\subsubsection{Results}

Figures \ref{fig:Linear} and \ref{fig:Banded} show boxplots of matching accuracy across $100$ simulation replications. Different aspect ratio, defined as the number of matched pairs divided by the covariate dimension $p$, have been used in the evaluation. In both scenarios---where the matching rule is correctly specified and where it is misspecified---RCA, DCA and the proposed semi-supervised method PAMA all benefit from higher aspect ratios, while Euclidean distance matching and propensity score matching show no improvements. In the correctly specified (linear) scenario, PAMA significantly outperforms other competitors, particularly as $p$ grows. In the thresholded (nonlinear) setting, although all algorithms are affected by misspecification, PAMA maintains competitive matching accuracy and robustness. Overall, PAMA exhibits two desirable properties:
\begin{itemize}
	\item Matching accuracy improves as the aspect ratio increases.
	\item Matching accuracy remains stable even as dimensionality increases, provided the aspect ratio is fixed.
\end{itemize}

\begin{figure}[h]
	\centering
	\includegraphics[width=.8\textwidth]{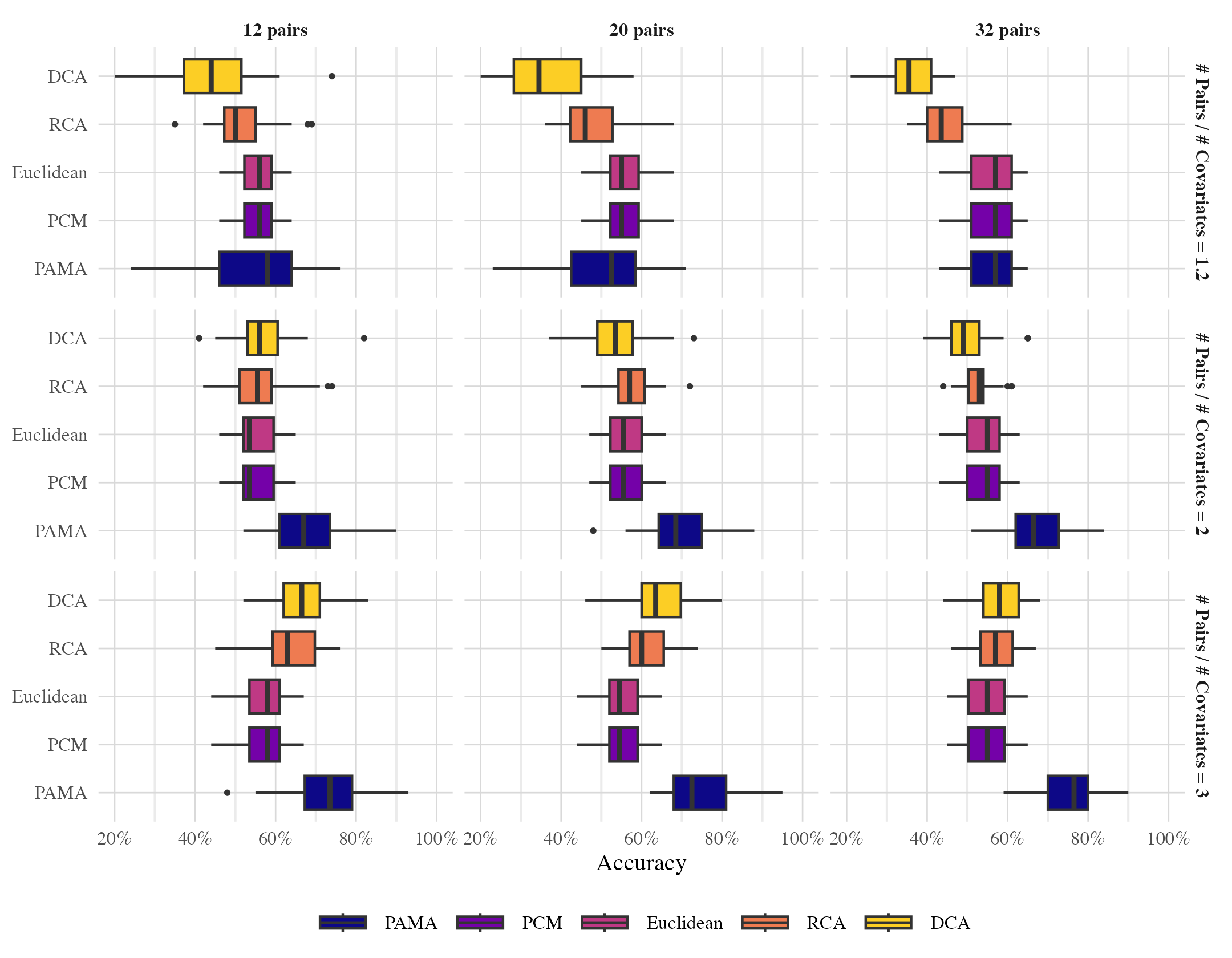}
	\caption{Matching Accuracy in Quadratic Score Setting}
	\label{fig:Linear}
\end{figure}

\begin{figure}[h]
	\centering
	\includegraphics[width=.8\textwidth]{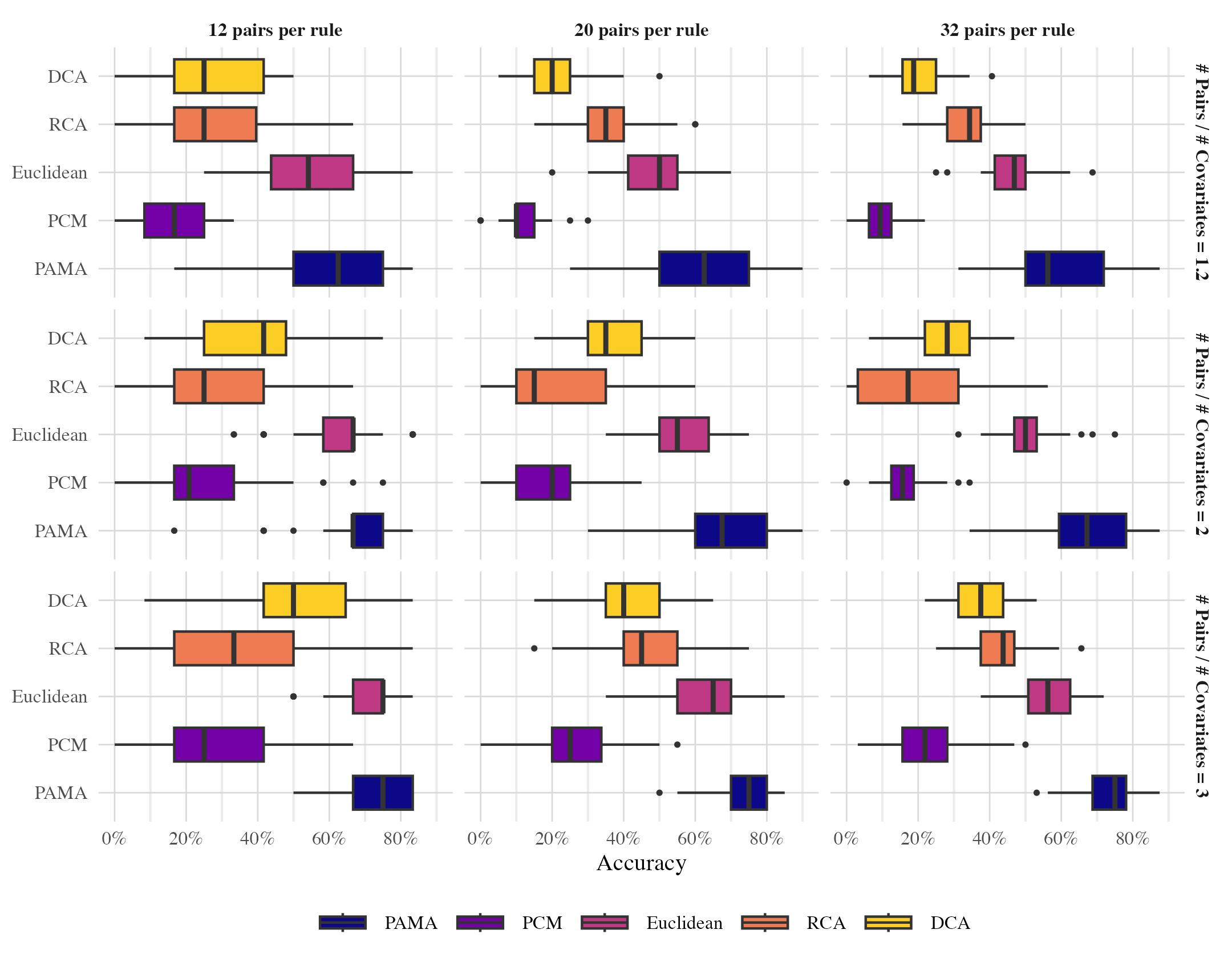}
	\caption{Matching Accuracy in Thresholding Rule Setting}
	\label{fig:Banded}
\end{figure}

\subsection{Algorithm Robustness}

We next investigate the robustness of PAMA to two types of model deviations: covariate correlation and interaction effects. In all settings, the number of covariates is fixed at $p=12$.

\begin{itemize}
	\item {\bf Covariate Correlation}: Correlations among covariates may obscure the contribution of the principal confounders, potentially leading the learned weights $\beta$ to diffuse across noisy variables. To examine this, we gradually increase the pairwise correlations among covariates while holding all other settings constant. As shown in Figure~\ref{fig:Correlation}, PAMA's matching accuracy remains stable across a wide range of correlation strengths, indicating resilience to moderate levels of collinearity.

\begin{figure}[h]
	\centering
	\includegraphics[width=.8\textwidth]{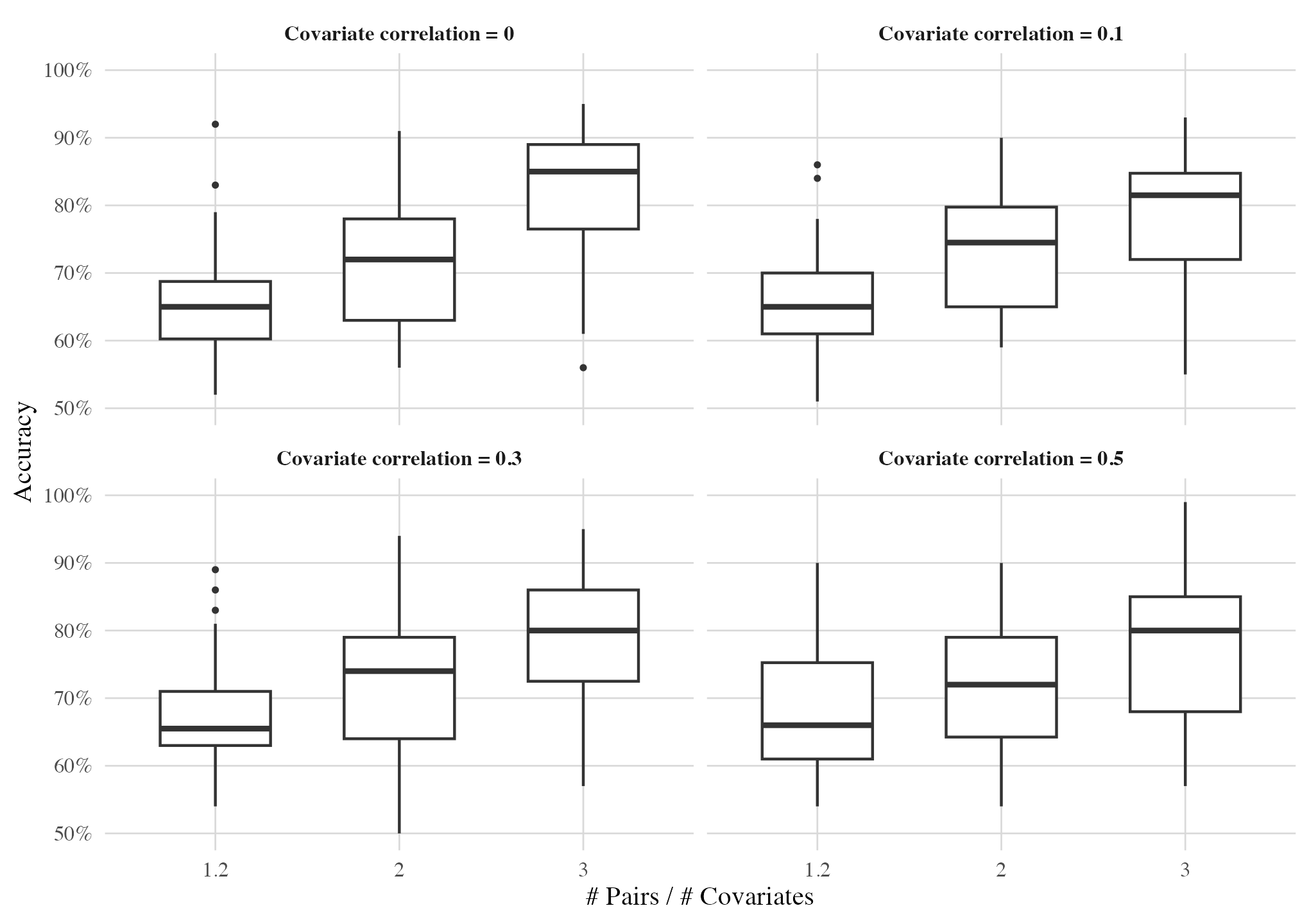}
	\caption{Matching accuracy under varying covariate correlation.}
	\label{fig:Correlation}
\end{figure}
	
	\item {\bf Interaction Effects}: To examine whether PAMA can detect signals arising from interactions between principal confounders and noisy covariates, we introduce second-order interaction terms between randomly selected pairs of these variables and incorporate them into the true matching logic. These terms are not included as explicit features, and thus act as unmodeled nonlinearities in the data-generating process.
	
	Figure~\ref{fig:Interaction} shows that interaction effects reduce the maximum attainable matching accuracy and slow the rate of improvement with additional training pairs. Nonetheless, PAMA continues to improve with more labeled data, demonstrating its capacity to model and learn from unspecified nonlinear structures like interaction-driven patterns.
	
	\begin{figure}[h]
		\centering
		\includegraphics[width=.8\textwidth]{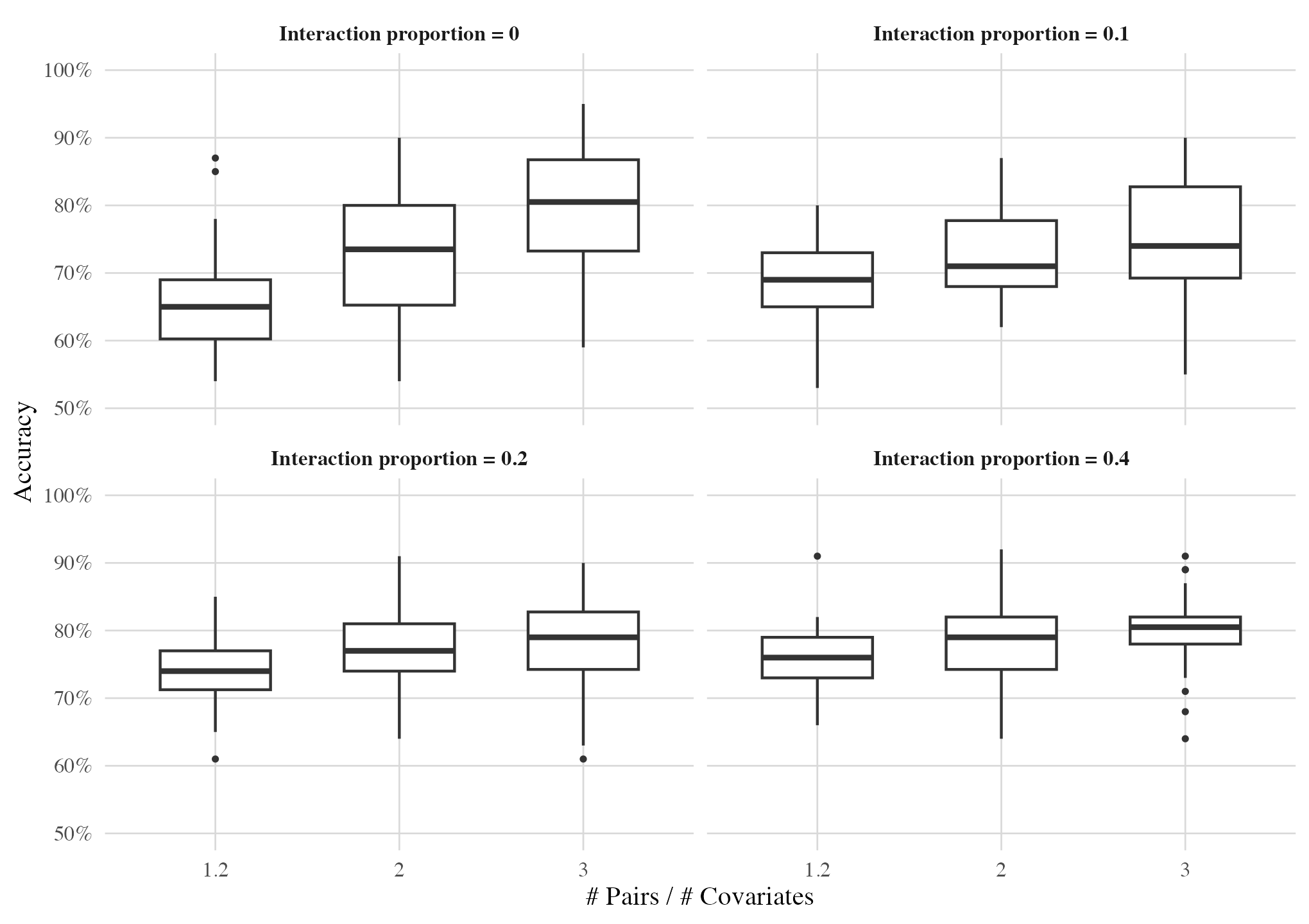}
		\caption{Matching accuracy in the presence of interaction effects.}
		\label{fig:Interaction}
	\end{figure}
	
	Furthermore, to understand how PAMA adjusts to these latent signals, we extract the estimated weights for the noisy covariates and compare the average magnitude of those involved in interaction terms to those that are not. The results, summarized in Table \ref{fig:interaction_weight}, show a consistent upward shift in the weights assigned to variables involved in interactions, suggesting that PAMA partially reallocates attention toward relevant but unmodeled structure.
	
	\begin{table}[h]
		\centering
		\begin{tabular}{|c|c|c|}
			\hline
			\# Training Pairs & \# Interactions & Average Weight Difference ($\times 10$) \\
			\hline
			15 & 1 & 1.39 \\
			15 & 2 & 1.04 \\
			15 & 3 & 0.44 \\
			15 & 5 & 0.41 \\
			\hline
			24 & 1 & 2.62 \\
			24 & 2 & 1.30 \\
			24 & 3 & 0.95 \\
			24 & 5 & 0.75 \\
			\hline
			36 & 1 & 2.66 \\
			36 & 2 & 1.17 \\
			36 & 3 & 1.64 \\
			36 & 5 & 0.86 \\
			\hline
		\end{tabular}
		\vskip 0.5em
		\caption{Average weight increase for interaction-involved variables.}
		\label{fig:interaction_weight}
	\end{table}
\end{itemize}

\subsection{Self-Taught Learning Evaluation}

Finally, we evaluate PAMA's self-taught learning extension. In the correctly specified scenario with 15 initial matched pairs, we vary the number of unmatched observations and track performance over 10 iterations, repeated across 50 trials. In each trial, the initial matching accuracy refers to the accuracy for unmatched observations using weights trained exclusively on matched pairs (i.e., weights obtained from (\ref{second opt})). The improvement in matching accuracy, as a function of initial accuracy, is recorded and presented in Figure \ref{fig:SelfTaught}. Notably, apart from rare cases of negative improvement, the gain curves exhibit a quadratic shape, with self-taught learning being most beneficial when initial performance is moderate.

To operationalize self-taught learning, we propose two complementary calibration strategies.
\begin{itemize}
	\item {\bf Simulation-Based Calibration}: Apply self-taught learning when the initial cross-validated performance falls within a moderate range, and terminate the procedure once accuracy exceeds a predetermined threshold.
	\item {\bf Transfer Learning Calibration}: When related but not identical pre-training data are available, transfer learning can be leveraged to inform the choice of the number of unmatched versus matched observations to include in the self-taught learning phase. Pretraining on related domains enables better estimation of gain curves and facilitates selection of the ratio that yields maximal improvement, potentially enhancing the quadratic-shaped benefits observed in self-taught learning.
\end{itemize}

\begin{figure}[h]
	\centering
	\includegraphics[width=.8\textwidth]{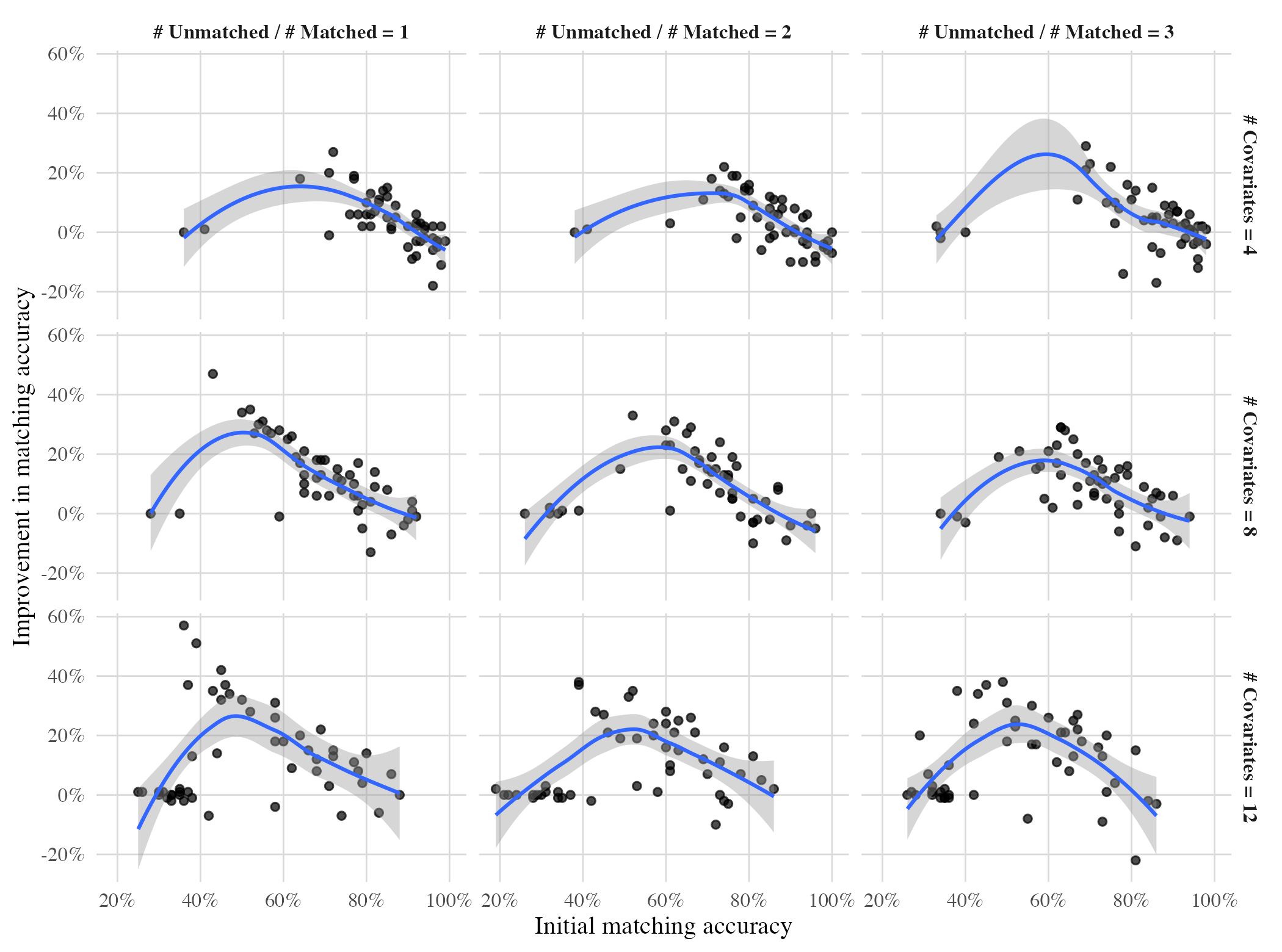}
	\caption{Self-Taught Learning Gains}
	\label{fig:SelfTaught}
\end{figure}


\section{Application of Matching in School Reopening Policy Evaluation}
\label{realdata}

We apply PAMA to one step of a real-world public health study aimed at estimating the association between in-person and virtual education and community COVID-19 case incidence following school reopenings during the first year of the COVID-19 pandemic. The study sample consists of 229 U.S. counties selected according to pre-specified inclusion criteria. From this sample, domain experts manually matched 51 pairs of counties based on their substantive knowledge of local policies and epidemiological contexts. Our goal is to evaluate how well PAMA can replicate the experts' pairing decisions using only the available baseline covariates. The covariates considered for matching include:
\begin{itemize}
	\item Geographical information (state, Bureau of Economic Analysis (BEA) region),
	\item The resumption of school district-level fall sports activity,
	\item Mask enforcement strength,
	\item Pre-intervention COVID-19 incidence rates,
	\item Additional socio-demographic and policy-related features,
\end{itemize}
for a total of 13 baseline variables.

Since PAMA requires numerical inputs, we preprocess categorical geographical indicators by extracting the first two principal coordinates via multidimensional scaling (MDS) \citep{MDS}, ensuring compatibility with the matching algorithm.

\subsection{Performance Evaluation}

We perform 5-fold cross-validation to compare PAMA with Euclidean distance matching, RCA, and DCA. In each fold, a subset of expert-paired observations is held out for evaluation, while the remaining pairs are used to train the matching algorithms. Given the limited test-set size, we reran the full training-evaluation pipeline 20 times with independent cross-validation splits to reduce random error. We report the mean matching accuracy, along with the standard deviation across runs in Table~\ref{tab:realdata_accuracy}. Notably, PAMA almost doubles the matching accuracy compared to baseline methods, demonstrating its ability to capture expert-driven pairing logic in a complex real-world setting.

\vskip 0.5em
\begin{table}[h]
	\centering
	\begin{tabular}{ |M{4cm}|M{2cm}|M{2cm}|M{2cm}|M{2cm}| }
		\hline
		& Euclidean Distance & RCA & DCA & PAMA \\
		\hline
		Matching Accuracy & 0.233 (0.045) & 0.257 (0.038) & 0.286 (0.040) & 0.395 (0.054) \\
		\hline
	\end{tabular}
	\vskip 0.5em
	\caption{Cross-validated Matching Accuracy for County Pairing in Studies of School Reopenings and COVID-19 Transmission}
	\label{tab:realdata_accuracy}
\end{table}

\subsection{Evaluation of Self-Taught Learning}

Beyond the 51 initially matched county pairs, 127 additional unmatched counties are available, 37 of which implemented virtual education. For the bipartite matching task, we therefore investigate whether the self-taught learning extension can improve performance by leveraging these unpaired counties.

	\begin{figure}[t]
		\centering
		\includegraphics[width=.8\textwidth,height=.35\textheight]{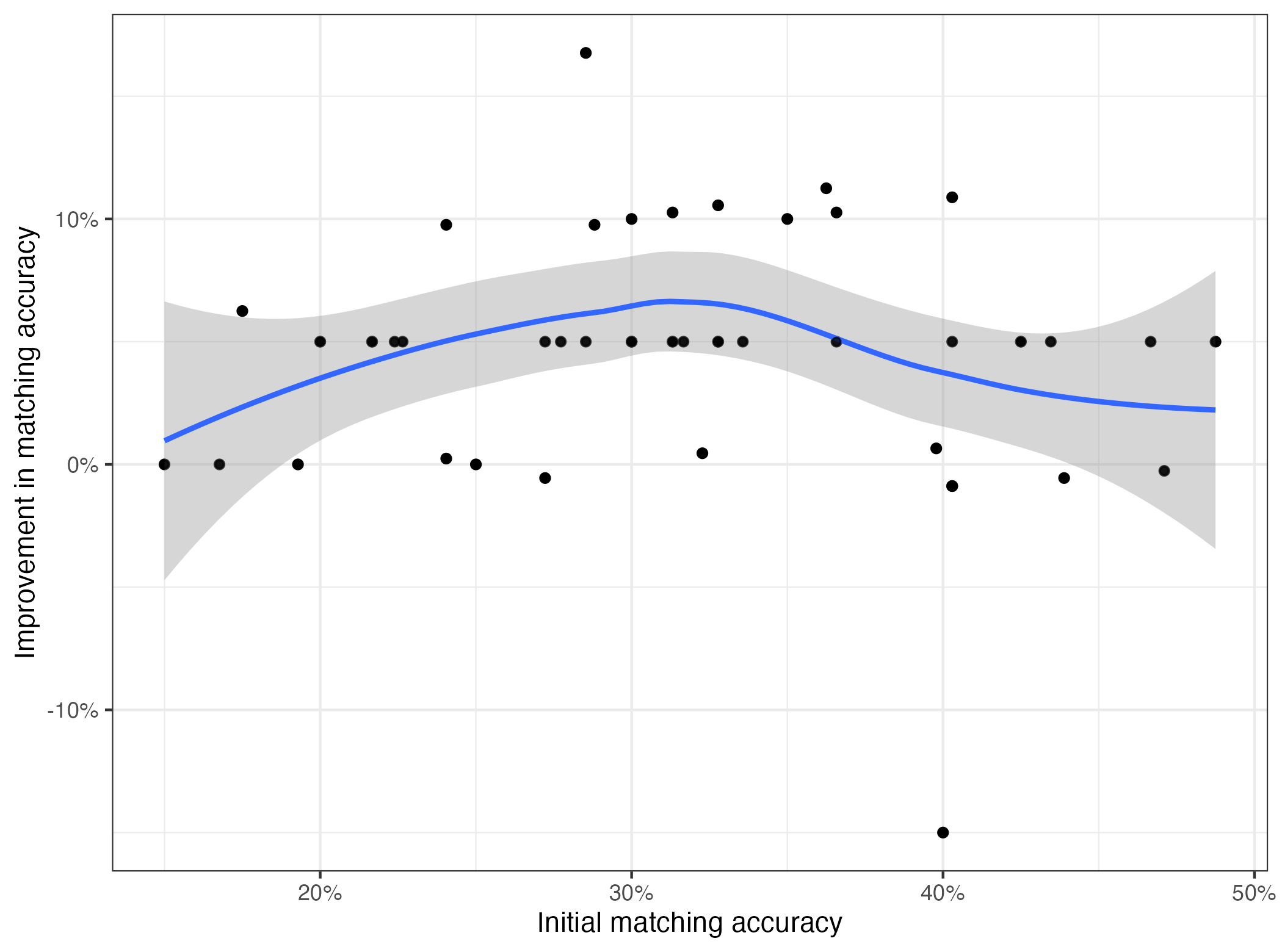}
		\caption{Self-taught learning on the county-matching dataset. Each point corresponds to one bootstrap iteration; the vertical axis shows the improvement in matching accuracy relative to the model fitted using only paired observations.}
		\label{fig:RealSemi}
	\end{figure}
	
For this simulation-based diagnostics pipeline, we draw a bootstrap sample with replacement from the 51 labeled pairs to form a training set, while the out-of-bag labeled pairs serve as the test set. The set of unlabeled counties used by self-taught learning is held fixed across each bootstrap replication. Figure \ref{fig:RealSemi} plots improvement versus baseline accuracy and shows a modest average gain over 50 replications. The result fits our impression from the simulation in Figure \ref{fig:SelfTaught}, indicating that self-taught learning provides limited additional benefit when baseline initial accuracy is already high or raletive low, but is helpful for most replications. Applying self-taught learning with the protocol above yields a mean (SD) initial matching accuracy of 0.399 (0.074) across the 50 replications. Therefore, as a result of this simulation-based diagnostic, these real-world experiments support the utility of the self-taught learning version of PAMA for recovering complex expert matching behavior and suggest that it is likely to provide additional benefits.
	
	

\section{Discussion}
\label{discussion}

This paper proposes the Priority-Aware one-to-one Matching Algorithm (PAMA), a semi-supervised framework that leverages limited expert-labeled pairs to learn covariate importance for matching in observational studies. PAMA addresses key challenges in policy evaluation---including high-dimensional covariates, difficulty in model specification, and data imbalance---by prioritizing covariates according to an approxiamted matching rule. Simulation studies and a real-world application on the impact of in-person schooling on COVID-19 transmission show that PAMA achieves higher matching accuracy and better recovers expert-driven matching logic than standard approaches, demonstrating robustness even when model assumptions are moderately violated. These findings highlight PAMA’s practical utility when expert manual labeling is valuable but resource-intensive, and the self-taught extension further enhances this utility by leveraging unpaired units, though simulation-based diagnostics may be needed for practical guidance. Altogether, PAMA offers a scalable and interpretable solution for matching in complex, policy-relevant datasets, and, to our knowledge, is the first framework to integrate expert knowledge through supervised metric learning in a semi-supervised context.

Despite its strong empirical performance, PAMA has certain limitations and opportunities for methodological refinement. First, PAMA assumes the matching rule can be approximated by a weighted Euclidean distance in the original feature space. While our simulations demonstrate robustness to this assumption, the framework could be extended by lifting the covariates into a higher-dimensional space, e.g., through inclusion of indicator functions, interaction terms, or employing a more flexible, kernelized version of PAMA, to better approximate thresholding or nonlinear expert matching logics. Second, the condition $p < \ell$ is required only for theoretical identifiability and interpretability of the weight vector $\beta$. If practitioners are interested solely in prediction rather than inference and interpretation, PAMA can be applied though this scenario was not evaluated in our simulations due to extensive computational burden. When variable importance is desired for interpretation, dimension reduction, regularization, or non-negativity constraints on $\beta$ may be helpful, especially in domains where negative weights are difficult to justify. Third, while matching accuracy is the primary evaluation metric, we did not provide a formal theoretical analysis of asymptotic matching error, as our framework is intentionally model-free. But under mild conditions, our consistency result for $\hat{\beta}$ implies that matching accuracy approaches one as the number of expert-labeled pairs increases. In practice, we believe simulation-based evaluation offers greater flexibility and practical guidance, particularly in settings where the expert matching process is only partially observed or nonlinear.

\begin{appendix}
\section{Collection of Technical Proofs}
In the appendix, we collect the proofs of lemmas, propositions and theorems presented in the paper, together with detailed explanation of some contents that might be tedious for the earlier sections.
        \subsection{Lemma and Propositions}
        \begin{proof}[Proof of Proposition \ref{prop1}]
        	Notice that both $L_w$ and $L_b$ are positive semidefinite and symmetric, so for the positive definite and symmetric matrix $\Sigma_w= \dot{X}^T L_w \dot{X} /\dot{\ell} +\lambda I $, we may decompose it to $\Sigma_w=\Gamma_{\sigma w}^T \Lambda_{\sigma w} \Gamma_{\sigma w}$ where $\Gamma_{\sigma w}$ is orthonormal and $\Lambda_{\sigma w}=\mathrm{diag}(\lambda_{w1},\cdots,\lambda_{wp})$ is a diagonal matrix with $\lambda_{wi} \geq 0$, $i=1,\cdots, p$. Define $\Lambda_{\sigma w}^{1/2} = \mathrm{diag}( \lambda_{w1}^{1/2},\cdots, \lambda_{wp}^{1/2} )$, $\Sigma_w^{1/2} = \Gamma_{\sigma w}^T \Lambda_{\sigma w}^{1/2} \Gamma_{\sigma w}$, and $\eta_w = \Sigma_{w}^{1/2}\beta / \Vert \Sigma_{w}^{1/2}\beta  \Vert_2  $, we have
        	\begin{equation*}
        		g(\beta) = \frac{ \beta^T \dot{X}^T L_b \dot{X} \beta / \dot{\ell}^2   }{\beta^T  \big(  \dot{X}^T L_w \dot{X} /\dot{\ell} +\lambda I   \big) \beta } = \f{1 }{ \dot{\ell}^2 } \eta^T_w \big( \Sigma_{w}^{-1/2} \big)^T \Sigma_b \Sigma_{w}^{-1/2} \eta_w.
        	\end{equation*}
        	Since the matrix $ \big( \Sigma_{w}^{-1/2} \big)^T\Sigma_b\Sigma_{w}^{-1/2}$ is again positive semidefinite and symmetric, we may decompose it to $\Gamma^T \Lambda \Gamma$ and define $\eta=\Gamma \eta_w$, where $\Gamma$ is an orthonormal matrix and $\Lambda=\mathrm{diag}(\lambda_{1},\cdots,\lambda_{p})$ is a diagonal matrix with $\lambda_{1} \geq \lambda_2 \geq \cdots \geq \lambda_{p} \geq 0$. Hence
        	\begin{equation}
        		g(\beta) = \f{1}{\dot{\ell}^2} \eta_w^T\Gamma^T \Lambda \Gamma \eta_w =  \f{1}{\dot{\ell}^2} \eta^T  \Lambda \eta \leq  \f{\lambda_1}{\dot{\ell}^2}  \Vert \eta \Vert_2^2 =   \f{\lambda_1}{\dot{\ell}^2} . \label{objectiveF}
        	\end{equation}
        	The equality sign of equation (\ref{objectiveF}) holds if and only if $\eta=(1,0,\cdots,0)^T$, which means $\eta_w$ is the eigenvector of $ \big( \Sigma_{w}^{-1/2} \big)^T\Sigma_b \Sigma_{w}^{-1/2}$ corresponding to its largest eigenvalue, and leads to the fact that $\beta$ is the eigenvector of $ \Sigma_{w}^{-1} \Sigma_b $ correspondingly to its largest eigenvalue.
        \end{proof}

        \begin{lemma}\label{CovConv} Under assumption \ref{assmp}, as $p, \dot{\ell}, \ddot{\ell} \to \infty$, the leading eigengap of $$\hat{\Sigma}_{\dot{\ell}}\triangleq \frac{1}{2\dot{\ell}} \sum_{k=1}^{\dot{\ell}} \big( \dot{x}_{c,k} - \dot{x}_{t,k} \big)  \big( \dot{x}_{c,k} - \dot{x}_{t,k} \big)^T $$ is bounded from below. Moreover,
        	\begin{equation*}
        		\Vert \hat{\Sigma}_{\dot{\ell}} - \Sigma \Vert_2 = O_p\Big( \sqrt{\frac{p}{\dot{\ell}}} \Big) , \quad {\rm and}\quad \Big\Vert \frac{\dot{X}^T \dot{X}}{2 \dot{\ell}} - \Sigma\Big\Vert_2 = O_p\Big(\sqrt {\frac{p}{\dot{\ell}}} \Big).
        	\end{equation*}
        	The same results hold when replacing $\hat{\Sigma}_{\dot{\ell}}$, $\dot{\ell}$ and $\dot{X}^T \dot{X}$ with $\hat{\Sigma}_{\ddot{\ell}}$, $\ddot{\ell}$ and $\ddot{X}^T \ddot{X}$ respectively.
        \end{lemma}
        \begin{proof}
        	Under assumption \ref{assmp}, the second statement follows directly from Theorem 6.5 of \citet{wainwright2019high}. We now prove the first statement.
        	\begin{align*}
        		\hat{\Sigma}_{\dot{\ell}} 
        		=& \frac{1}{2\dot{\ell}} \sum_{k=1}^{\dot{\ell}} \Big[ \big( \dot{x}_{c,k} - \dot{x}_{t,k} \big)  \big( \dot{x}_{c,k} - \dot{x}_{t,k} \big)^T - 2 \Sigma \Big]+  \Sigma := K +  \Sigma
        	\end{align*}
        	By Weyl's Inequality and Theorem 6.5 of \citet{wainwright2019high}, we conclude that
        	\begin{equation}
        		\lambda_1(\hat{ \Sigma }_{\dot{\ell}}) -  \lambda_2(\hat{ \Sigma }_{\dot{\ell}}) \geq \big( \lambda_1 - \lambda_2 \big)  - 2||K||_2  = \big( \lambda_1 - \lambda_2 \big) - O_p(\sqrt{p / \dot{\ell}}) \label{weyl}
        	\end{equation}
        	where $\lambda_1(\cdot)$ and $\lambda_2(\cdot)$ denotes the largest and second largest eigenvalues of the corresponding $p\times p$ matrix, respectively, and since $p=o(\dot{\ell})$, (\ref{weyl}) implies that the smallest eigenvalue of $\hat{ \Sigma }_{\dot{\ell}}$ is bounded away from zero.
        \end{proof}
        
        \begin{lemma} \label{lemmaDis}
        	Suppose $P\triangleq ( P_0, P_1 )$ and $Q\triangleq ( Q_{0}, Q_1 )$ are square orthonormal matrices. where $P_0, Q_0 \in \mathcal{M}_{p\times 1}$ and $P_1, Q_1 \in \mathcal{M}_{p\times(p-1)}$. Then for $\mathrm{dist}(P_0, Q_0)$ defined as $\mathrm{dist}(P_0, Q_0) \triangleq  \Vert P_0 P_0^T - Q_0 Q_0^T \Vert_2 $, we have
        	\begin{equation*}
        		\mathrm{dist}(P_0, Q_0) =\Vert P_1^T Q_0 \Vert_2 = \Vert Q_1^T P_0 \Vert_2
        	\end{equation*}
        \end{lemma}
        \begin{proof}[Proof of Lemma \ref{lemmaDis}]
        	Define $P_0^T Q_0 = \cos \theta$ with some $\theta \in [0,\pi]$ since $P_0$ and $Q_0$ are unit vectors. Thus
        	\begin{align}
        		\Vert P_1^T Q_0 \Vert_2 &= \left(Q_0^T P_1 P_1^T Q_0  \right)^{1/2} = \left(  Q_0^T \big( I - P_0 P_0^T \big) Q_0  \right)^{1/2} \label{RHS} \\
        		&= (1-\cos^2\theta)^{1/2} =\sin\theta . \nn 
        	\end{align}
        	The singular value decomposition further implies that, there exist an orthonormal matrix $\tilde{U}$ such that, $Q_0^T P_1 = (\sin\theta, 0,\cdots, 0) \cdot \tilde{U}^T$. Therefore,
        	\begin{align*}
        		P_0 P_0^T - Q_0 Q_0^T &= P_0 P_0^T -  ( P_0, P_1 )  \begin{pmatrix}
        			P_0^T \\
        			P_1^T
        		\end{pmatrix} Q_0 Q_0^T  ( P_0, P_1 )  \begin{pmatrix}
        			P_0^T \\
        			P_1^T
        		\end{pmatrix} \\
        		&= P_0 P_0^T -  ( P_0, P_1 )  \begin{pmatrix}
        			\cos\theta \\
        			\tilde{U}\cdot \begin{pmatrix}
        				\sin\theta \\
        				0_{(p-1)\times 1}
        			\end{pmatrix} 
        		\end{pmatrix}  \begin{pmatrix}
        			\cos\theta, (\sin\theta,
        			0_{1\times (p-1)}) \cdot
        			\tilde{U}^T
        		\end{pmatrix}   \begin{pmatrix}
        			P_0^T \\
        			P_1^T
        		\end{pmatrix} \\
        		&= ( P_0, P_1 )  \begin{pmatrix}
        			1 & 0\\
        			0& \tilde{U}
        		\end{pmatrix}  \begin{pmatrix}
        			\begin{pmatrix}
        				\sin^2\theta & -\cos\theta\sin\theta\\
        				-\cos\theta\sin\theta &  -\sin^2\theta
        			\end{pmatrix}  & 0_{2\times (p-2)}\\
        			0_{(p-2)\times 2}& 0
        		\end{pmatrix}  \begin{pmatrix}
        			1 & 0\\
        			0& \tilde{U}^T
        		\end{pmatrix}  \begin{pmatrix}
        			P_0^T \\
        			P_1^T
        		\end{pmatrix}.
        	\end{align*}
        	Hence
        	\begin{equation}
        		\Vert P_0 P_0^T - Q_0 Q_0^T \Vert_2 =\Big\Vert  \begin{pmatrix}
        			\sin^2\theta & -\cos\theta\sin\theta\\
        			-\cos\theta\sin\theta &  -\sin^2\theta
        		\end{pmatrix}  \Big \Vert_2 =\sin \theta \label{LHS}
        	\end{equation}
        	Combine (\ref{LHS}) with equation (\ref{RHS}) leads to the result and we proved the lemma.
        \end{proof}
        
        \subsection{Theorems}
        
        \begin{theorem}[Restatement of the Davis-Kahan $\sin(\theta)$ Theorem]\label{DKTheorem}
        	Consider two symmetric positive definite matrix $G, \hat{G} = G + H \in \mathcal{M}_{p \times p}$, with eigen-decompositions
        	\begin{align*}
        		G & = U\Xi U^T =\begin{pmatrix}
        			U_0 & U_1
        		\end{pmatrix} \begin{pmatrix}
        			\xi_1 & 0 \\ 0 &\Xi_2
        		\end{pmatrix}\begin{pmatrix}
        			U_0^T \\ U_1^T
        		\end{pmatrix} \\
        		\hat{G} & = \hat{U} \hat{\Xi} \hat{U}^T  = \begin{pmatrix}
        			\hat{U}_0 & \hat{U}_1
        		\end{pmatrix} \begin{pmatrix}
        			\hat{\xi}_1 & 0 \\ 0 & \hat{\Xi}_2
        		\end{pmatrix}\begin{pmatrix}
        			\hat{U}_0^T \\ \hat{U}_1^T
        		\end{pmatrix}
        	\end{align*}
        	where $\xi_1 \geq \cdots \geq \xi_p$ are the ordered eigenvalues of $G$ and $\Xi_2 = \mathrm{diag}(\xi_2,\cdots,\xi_p)$; $\hat{\xi}_1$ and $  \hat{\Xi}_2$ are similarly defined. Then
        	\begin{equation*}
        		\mathrm{dist}(U_0,\hat{U}_0) = \big\Vert  U_0U_0^T - \hat{U}_0\hat{U}_0^T \big\Vert_2 \leq  \frac{ \Vert H  \Vert_2 }{\xi_1-\xi_2  - \Vert H \Vert_2}.
        	\end{equation*}
        	if the denominator is bounded away from zero.
        \end{theorem}
        \begin{proof}[Proof of theorem \ref{DKTheorem}]
        	Notice that $\Vert \hat{\Xi}_2 - \Xi_2 \Vert_2 \leq \Vert H \Vert_2$ by Weyl's theorem, thus,
        	\begin{align*}
        		\big\Vert \hat{U}_1^T H U_0 \big\Vert_2 =& \big\Vert \hat{U}_1^T (\hat{U} \hat{\Xi} \hat{U}^T - U \Xi U^T) U_0 \big\Vert_2= \big\Vert \hat{\Xi}_2 \hat{U}_1^T U_0 -  \hat{U}_1^T U_0 \xi_1 \big\Vert_2 \\
        		\geq &\big\Vert \hat{U}_1^T U_0 \xi_1\big\Vert_2  - \big\Vert \Xi_2 \hat{U}_1^T U_0\big\Vert_2 - \big\Vert  (\hat{\Xi}_2 - \Xi_2 ) \hat{U}_1^T U_0 \big\Vert_2 \\
        		\geq &\big\Vert \hat{U}_1^T U_0 \big\Vert_2 \xi_1 - \big\Vert \hat{U}_1^T U_0 \big\Vert_2 \big( \big\Vert \Xi_2 \big\Vert_2 +   \big\Vert \hat{\Xi}_2 - \Xi_2 \big\Vert_2 \big) \\
        		\geq & \big\Vert \hat{U}_1^T U_0 \big\Vert_2 \cdot \big( \xi_1 - \xi_2 - \Vert H \Vert_2  \big).
        	\end{align*}
        	Therefore, by Lemma \ref{lemmaDis}, we have
        	\begin{equation*}
        		\mathrm{dist}(U_0,\hat{U}_0) =\Vert \hat{U}_1^T U_0 \Vert_2 \leq \frac{ \Vert \hat{U}_1^T H U_0\Vert_2}{ \xi_1 - \xi_2  - \Vert H \Vert_2} \leq \frac{\Vert \hat{U}_1 \Vert_2 \cdot \Vert  H \Vert_2 \cdot \Vert U_0 \Vert_2 }{ \xi_1 - \xi_2  - \Vert H \Vert_2} \leq  \frac{ \Vert H  \Vert_2 }{ \xi_1 - \xi_2   - \Vert H \Vert_2}.
        	\end{equation*}
        \end{proof}
   
        \begin{proof}[Proof of Theorem \ref{main_theorem}]
        	Define $\mathcal{S}(Z)=\{ z_{c,i}, z_{t,i}\in Z: (z_{c,i},z_{t,i}) \}_{i\geq 1}$ be the set of matching pairs of $Z$, and correspondingly, define $G(Z)= \Sigma_{\omega Z}^{-1}\Sigma_{bZ}$, where
        	\begin{align*}
        		\Sigma_{\omega Z}&= \lambda  I + \frac{1}{\dot{\ell}} \sum_{(z_{c,i},z_{t,i})\in\mathcal{S}(Z)} (z_{c,i}-z_{t,i})(z_{c,i}-z_{t,i})^T ,\\
        		\Sigma_{bZ} &=  \sum_{z_c\in Z_c, z_t\in Z_t, (z_c,z_t)\notin\mathcal{S}(Z)} (z_c-z_t)(z_c-z_t)^T.
        	\end{align*}
        	Then according to Proposition \ref{prop1}, $\hat{\beta}^{(0)}$ is the eigenvector of $G(\dot{X})$ corresponding to its largest eigenvalue and $\hat{\beta}^{*}$ is the eigenvector of $G(X)$ corresponding to its largest eigenvalue. Therefore, according to Davis-Kahan $\sin( \theta)$ theorem, we only need to investigate the order of $\Vert H\Vert_2= \Vert G(\dot{X}) - G(X) \Vert_2$ and the leading eigengap of $G(\dot{X})$.
        	\begin{enumerate}[label=(\roman*)]
        		\item {For the leading eigengap of $G(\dot{X})$}: 
        		\vskip 0.5em
        		Notice that
        		\begin{align}
        			G(\dot{X}) &= \frac{1}{\dot{\ell}}\Big(\dot{X}^T L_w \dot{X} + \lambda \dot{\ell} I \Big)^{-1} \dot{X}^T \Big(-L_w +  \dot{\ell} I +E \Big) \dot{X}  \label{1}\\
        			&= -\frac{1}{\dot{\ell}}\Big(\dot{X}^T L_w \dot{X} + \lambda \dot{\ell} I \Big)^{-1} \dot{X}^T L_w \dot{X} + \Big( \dot{X}^T L_w \dot{X} + \lambda \dot{\ell} I \Big)^{-1} \dot{X}^T \dot{X} \nn  \\
        			& \qquad +  \Big( \dot{X}^T L_w \dot{X} + \lambda \dot{\ell} I \Big)^{-1} \frac{1}{\dot{\ell}}\dot{X}^T E \dot{X} \triangleq G_1(\dot{X}) + G_2(\dot{X}) + G_3(\dot{X}), \nn
        		\end{align}
        		where 
        		\begin{equation*}
        			E= \begin{pmatrix}
        				&0   & \mathbbm{1}\cdot \mathbbm{1}^T \\
        				&\mathbbm{1} \cdot \mathbbm{1}^T  & 0  
        			\end{pmatrix}, \quad {\rm with}\;\; \mathbbm{1} = (1,\cdots,1)^T \in \mathcal{M}_{\dot{\ell}\times 1}.
        		\end{equation*}
        		By central limit theorem, it's trivial that
        		\begin{equation*}
        			\Big \Vert \frac{1}{\sqrt{\dot{\ell}}}\mathbbm{1}^T \dot{X}_c \Big \Vert_2 =O_p(\sqrt{p}), \;\; {\rm and} \;\; \Big \Vert \frac{1}{\sqrt{\dot{\ell}}}\mathbbm{1}^T \dot{X}_t \Big \Vert_2 =O_p(\sqrt{p}).
        		\end{equation*}
        		Therefore,
        		\begin{align}
        			\Vert G_3(\dot{X}) \Vert_2  &= \Big\Vert  \Big( \dot{X}^T L_w \dot{X} + \lambda \dot{\ell} I \Big)^{-1} \frac{1}{\dot{\ell}}\dot{X}^T E \dot{X} \Big \Vert_2  \label{2} \\
        			& = \Big\Vert  \Big( \dot{X}^T L_w \dot{X} + \lambda \dot{\ell} I \Big)^{-1} \frac{1}{\dot{\ell}}\Big(  \dot{X}_c^T \mathbbm{1} \mathbbm{1}^T \dot{X}_t + \dot{X}_t^T \mathbbm{1} \mathbbm{1}^T \dot{X}_c \Big)\Big \Vert_2  \nn \\
        			& \leq \frac{2}{\lambda \dot{\ell}} \Big \Vert \frac{1}{\sqrt{\dot{\ell}}}\mathbbm{1}^T \dot{X}_c \Big \Vert_2 \cdot \Big \Vert \frac{1}{\sqrt{\dot{\ell}}}\mathbbm{1}^T \dot{X}_t \Big \Vert_2  = O_p \Big( \frac{p}{\dot{\ell}}  \Big). \nn
        		\end{align}
        		Because 
        		\begin{gather*}
        			\big\Vert 	\hat{\Sigma}_{\dot{\ell}} -\Sigma \big\Vert_2  = \Big\Vert \frac{1}{2\dot{\ell}} \dot{X}^T L_w \dot{X} -\Sigma \Big\Vert_2 = \Big\Vert \frac{1}{2\dot{\ell}} \sum_{i=1}^{\dot{\ell}} (\dot{x}_{c,i}-\dot{x}_{t,i})(\dot{x}_{c,i}-\dot{x}_{t,i})^T -  \Sigma \Big\Vert_2= O_p\Big( \sqrt{\frac{p}{\dot{\ell}}} \Big), \\
        			{\rm and}\quad  \Big\Vert \frac{1}{2\dot{\ell}} \dot{X}^T \dot{X} - \Sigma \Big\Vert_2 = O_p\Big( \sqrt{\frac{p}{\dot{\ell}}} \Big) ,
        		\end{gather*}
        		according to Lemma \ref{CovConv}. Thus, by Weyl's theorem, we have
        		\begin{equation}
        			\Vert G_1(\dot{X}) \Vert_2 \leq \frac{1}{\dot{\ell}}  \frac{\lambda_1+O_p\Big( \sqrt{\frac{p}{\dot{\ell}}} \Big) }{\lambda_p + \frac{\lambda}{2} - O_p\Big( \sqrt{\frac{p}{\dot{\ell}}} \Big)}= O_p \Big( \frac{1+\lambda_1}{\dot{\ell}} \Big). \label{3}
        		\end{equation}
        		Meanwhile, we further decompose $G_2(\dot{X}) $,
        		\begin{align}
        			G_2(\dot{X}) &= \Big( \frac{1}{2\dot{\ell}}\dot{X}^T L_w \dot{X} + \f{\lambda}{2} I \Big)^{-1} \Big( \frac{1}{2\dot{\ell}} \dot{X}^T \dot{X} \Big) \label{4}\\
        			&   = \Big[   \Big( \frac{1}{2\dot{\ell}}\dot{X}^T L_w \dot{X} + \f{\lambda}{2} I \Big)^{-1} -    \Big(\Sigma + \f{\lambda}{2} I \Big)^{-1}   \Big]\Big( \frac{1}{2\dot{\ell}} \dot{X}^T \dot{X} \Big) \nn \\
        			& \quad + \Big(\Sigma + \f{\lambda}{2} I \Big)^{-1} \Big( \frac{1}{2\dot{\ell}} \dot{X}^T \dot{X} - \Sigma \Big)  + \Big(\Sigma + \f{\lambda}{2} I \Big)^{-1} \Sigma \nn \\
        			&  \triangleq G_4(\dot{X})+ G_5(\dot{X}) + G_6. \nn
        		\end{align}
        		Since
        		\begin{align*}
        			G_4(\dot{X})& =  \Big( \frac{1}{2\dot{\ell}}\dot{X}^T L_w \dot{X} + \f{\lambda}{2} I \Big)^{-1}  \left[ \Big( \Sigma + \f{\lambda}{2} I \Big)  -  \Big( \frac{1}{2\dot{\ell}}\dot{X}^T L_w \dot{X} + \f{\lambda}{2} I  \Big)  \right]  \Big(\Sigma + \f{\lambda}{2} I \Big)^{-1}  \\
        			& \qquad \cdot \Big[ \Sigma + \Big(   \frac{1}{2\dot{\ell}} \dot{X}^T \dot{X} - \Sigma   \Big) \Big] \\
        			& = \Big(  	\hat{\Sigma}_{\dot{\ell}}  + \f{\lambda}{2} I \Big)^{-1}  \Big( \Sigma - 	\hat{\Sigma}_{\dot{\ell}}   \Big)  \Big(\Sigma + \f{\lambda}{2} I \Big)^{-1} \Big[ \Sigma + \Big(   \frac{1}{2\dot{\ell}} \dot{X}^T \dot{X} - \Sigma   \Big) \Big],
        		\end{align*}
        		so
        		\begin{align}
        			\Vert G_4(\dot{X}) \Vert_2& \leq \Big \Vert  \Big(  	\hat{\Sigma}_{\dot{\ell}}  + \f{\lambda}{2} I \Big)^{-1}  \Big( \Sigma - 	\hat{\Sigma}_{\dot{\ell}}   \Big)  \Big(\Sigma + \f{\lambda}{2} I \Big)^{-1} \Big[ \Sigma + \Big(   \frac{1}{2\dot{\ell}} \dot{X}^T \dot{X} - \Sigma   \Big) \Big] \Big \Vert_2 \label{5} \\
        			& \leq \frac{O_p\Big( \sqrt{\frac{p}{\dot{\ell}}} \Big)}{\lambda_{p}+\f{\lambda}{2} - O_p\Big( \sqrt{\frac{p}{\dot{\ell}}} \Big)} \cdot \frac{\Big( \lambda_1 + O_p\Big( \sqrt{\frac{p}{\dot{\ell}}} \Big) \Big)}{ \lambda_{p}+\f{\lambda}{2} }  = O_p\Big( (1+\lambda_1) \sqrt{\frac{p}{\dot{\ell}}} \Big) . \nn
        		\end{align}
        		Similarly,
        		\begin{equation}
        			\Vert G_5(\dot{X}) \Vert_2  \leq \frac{ O_p\Big( \sqrt{\frac{p}{\dot{\ell}}}  \Big)}{ \lambda_{p}+\f{\lambda}{2} }  = O_p\Big( \sqrt{\frac{p}{\dot{\ell}}}  \Big). \label{6}
        		\end{equation}
        		Because $\Sigma$ is symmetric and positive definite, so we have the decomposition $\Sigma=U_\Sigma \Lambda U_{\Sigma}^T$, where $\Lambda=\mathrm{diag}(\lambda_1,\cdots,\lambda_p)$ and $U_{\Sigma}$ is a orthonormal matrix. Hence
        		\begin{equation*}
        			G_6 =  U_{\Sigma} \cdot \mathrm{diag} \Big(  \frac{\lambda_1}{\lambda_1+\f{\lambda}{2}} , \cdots,  \frac{\lambda_p}{\lambda_p+\f{\lambda}{2}}   \Big) \cdot 
        			U_{\Sigma}^T
        		\end{equation*}
        		since $x/(x+\lambda/2)$ is a increasing function of $x$, so the leading eigengap of $G_6$ is
        		\begin{equation}
        			\frac{\lambda_1}{\lambda_1+\f{\lambda}{2}} - \frac{\lambda_2}{\lambda_2+\f{\lambda}{2}} = \frac{\lambda (\lambda_1-\lambda_2)}{2 \big(\lambda_1+\f{\lambda}{2}\big)\big(\lambda_2+\f{\lambda}{2}\big) } = O\Big(  \f{1}{(1+\lambda_1)^2} \Big). \label{7}
        		\end{equation}
        		Combine (\ref{1})$\sim$(\ref{7}), by Weyl's theorem, we conclude that the leading eigengap of $G(\dot{X})$, denoted as $\tilde{\lambda}_1- \tilde{\lambda}_2$, satisfy
        		\begin{equation}
        			\tilde{\lambda}_1- \tilde{\lambda}_2 \geq  O\Big(  \f{1}{(1+\lambda_1)^2} + (1+\lambda_1) \sqrt{\frac{p}{\dot{\ell}}} \Big)= O\Big(  \f{1}{(1+\lambda_1)^2}  \Big). \label{eigengap}
        		\end{equation}
        		\item {For the perturbation $H = G(\dot{X}) - G(X)$}: 
        		\vskip 0.5em
        		With same decomposition of (\ref{1}) and $(\ref{4})$, we have
        		\begin{equation*}
        			H= G(\dot{X}) - G(X) = \sum_{i\in \{ 1,3,4,5 \}} \big( G_i(\dot{X}) - G_i(X) \big)
        		\end{equation*}
        		where $G_6$ cancels out. Because $X$ behave similar to $\dot{X}$ except for having a larger sample size (number of rows), (\ref{2}), (\ref{3}), (\ref{5}) and (\ref{6}) remian true when replacing $\dot{X}$ and $\dot{\ell}$ with $X$ and $\dot{\ell}+\ddot{\ell}$, respectively. Therefore,
        		\begin{align}
        			\Vert H \Vert_2 &\leq \sum_{i\in \{ 1,3,4,5 \}}  \Big( \Vert G_i(\dot{X}) \Vert_2 + \Vert G_i(X) \Vert_2   \Big)\label{perturbation} \\
        			&= O\Big(  (1+\lambda_1) \sqrt{\frac{p}{\dot{\ell}}} \Big) + O\Big(  (1+\lambda_1) \sqrt{\frac{p}{\dot{\ell}+\ddot{\ell}}} \Big) = O\Big(  (1+\lambda_1) \sqrt{\frac{p}{\dot{\ell}}} \Big) . \nn
        		\end{align}
        		Therefore, according to Davis-Kahan $\sin( \theta)$ theorem and combine (\ref{eigengap}), (\ref{perturbation}), we conclude that
        		\begin{equation*}
        			\mathrm{dist}(\hat{\beta}^{(0)}, \hat{\beta}^*) \leq \frac{ \Vert H \Vert_2  }{  \tilde{\lambda}_1- \tilde{\lambda}_2 - \Vert H \Vert_2  } = O\Big(  (1+\lambda_1)^3 \sqrt{\frac{p}{\dot{\ell}}} \Big).
        		\end{equation*}
        	\end{enumerate}
        \end{proof}
\end{appendix}

\bibliographystyle{imsart-nameyear}
\bibliography{references.bib}

\begin{thebibliography}{18}

\bibitem[\protect\citeauthoryear{Austin}{2014}]{austin2014comparison}
\begin{barticle}[author]
\bauthor{\bsnm{Austin},~\bfnm{Peter~C}\binits{P.~C.}}
(\byear{2014}).
\btitle{A comparison of 12 algorithms for matching on the propensity score}.
\bjournal{Statistics in medicine}
\bvolume{33}
\bpages{1057--1069}.
\end{barticle}
\endbibitem

\bibitem[\protect\citeauthoryear{Bar-Hillel et~al.}{2005}]{RCA}
\begin{barticle}[author]
\bauthor{\bsnm{Bar-Hillel},~\bfnm{Aharon}\binits{A.}},
  \bauthor{\bsnm{Hertz},~\bfnm{Tomer}\binits{T.}},
  \bauthor{\bsnm{Shental},~\bfnm{Noam}\binits{N.}},
  \bauthor{\bsnm{Weinshall},~\bfnm{Daphna}\binits{D.}} \AND
  \bauthor{\bsnm{Ridgeway},~\bfnm{Greg}\binits{G.}}
(\byear{2005}).
\btitle{Learning a Mahalanobis metric from equivalence constraints.}
\bjournal{Journal of machine learning research}
\bvolume{6}.
\end{barticle}
\endbibitem

\bibitem[\protect\citeauthoryear{Cox and Cox}{2008}]{MDS}
\begin{bincollection}[author]
\bauthor{\bsnm{Cox},~\bfnm{Michael~AA}\binits{M.~A.}} \AND
  \bauthor{\bsnm{Cox},~\bfnm{Trevor~F}\binits{T.~F.}}
(\byear{2008}).
\btitle{Multidimensional scaling}.
In \bbooktitle{Handbook of data visualization}
\bpages{315--347}.
\bpublisher{Springer}.
\end{bincollection}
\endbibitem

\bibitem[\protect\citeauthoryear{Fisher}{1936}]{fisher1936use}
\begin{barticle}[author]
\bauthor{\bsnm{Fisher},~\bfnm{Ronald~A}\binits{R.~A.}}
(\byear{1936}).
\btitle{The use of multiple measurements in taxonomic problems}.
\bjournal{Annals of eugenics}
\bvolume{7}
\bpages{179--188}.
\end{barticle}
\endbibitem

\bibitem[\protect\citeauthoryear{Gareth et~al.}{2023}]{gareth2023introduction}
\begin{bbook}[author]
\bauthor{\bsnm{Gareth},~\bfnm{James}\binits{J.}},
  \bauthor{\bsnm{Witten},~\bfnm{Daniela}\binits{D.}},
  \bauthor{\bsnm{Hastie},~\bfnm{Trevor}\binits{T.}},
  \bauthor{\bsnm{Tibshirani},~\bfnm{Robert}\binits{R.}} \AND
  \bauthor{\bsnm{Taylor},~\bfnm{Jonathan}\binits{J.}}
(\byear{2023}).
\btitle{An introduction to statistical learning: With applications in python}.
\bpublisher{Springer International Publishing: Cham, Switzerland}.
\end{bbook}
\endbibitem

\bibitem[\protect\citeauthoryear{Globerson and
  Roweis}{2005}]{globerson2005metric}
\begin{barticle}[author]
\bauthor{\bsnm{Globerson},~\bfnm{Amir}\binits{A.}} \AND
  \bauthor{\bsnm{Roweis},~\bfnm{Sam}\binits{S.}}
(\byear{2005}).
\btitle{Metric learning by collapsing classes}.
\bjournal{Advances in neural information processing systems}
\bvolume{18}.
\end{barticle}
\endbibitem

\bibitem[\protect\citeauthoryear{Hastie, Buja and
  Tibshirani}{1995}]{hastie1995penalized}
\begin{barticle}[author]
\bauthor{\bsnm{Hastie},~\bfnm{Trevor}\binits{T.}},
  \bauthor{\bsnm{Buja},~\bfnm{Andreas}\binits{A.}} \AND
  \bauthor{\bsnm{Tibshirani},~\bfnm{Robert}\binits{R.}}
(\byear{1995}).
\btitle{Penalized discriminant analysis}.
\bjournal{The Annals of Statistics}
\bvolume{23}
\bpages{73--102}.
\end{barticle}
\endbibitem

\bibitem[\protect\citeauthoryear{Hastie, Tibshirani and
  Buja}{1994}]{hastie1994flexible}
\begin{barticle}[author]
\bauthor{\bsnm{Hastie},~\bfnm{Trevor}\binits{T.}},
  \bauthor{\bsnm{Tibshirani},~\bfnm{Robert}\binits{R.}} \AND
  \bauthor{\bsnm{Buja},~\bfnm{Andreas}\binits{A.}}
(\byear{1994}).
\btitle{Flexible discriminant analysis by optimal scoring}.
\bjournal{Journal of the American statistical association}
\bvolume{89}
\bpages{1255--1270}.
\end{barticle}
\endbibitem

\bibitem[\protect\citeauthoryear{Hastie and
  Tibshirani}{1995}]{hastie1995discriminant}
\begin{barticle}[author]
\bauthor{\bsnm{Hastie},~\bfnm{Trevor}\binits{T.}} \AND
  \bauthor{\bsnm{Tibshirani},~\bfnm{Robert}\binits{R.}}
(\byear{1995}).
\btitle{Discriminant adaptive nearest neighbor classification and regression}.
\bjournal{Advances in neural information processing systems}
\bvolume{8}.
\end{barticle}
\endbibitem

\bibitem[\protect\citeauthoryear{Hastie and
  Tibshirani}{1996}]{hastie1996discriminant}
\begin{barticle}[author]
\bauthor{\bsnm{Hastie},~\bfnm{Trevor}\binits{T.}} \AND
  \bauthor{\bsnm{Tibshirani},~\bfnm{Robert}\binits{R.}}
(\byear{1996}).
\btitle{Discriminant analysis by Gaussian mixtures}.
\bjournal{Journal of the Royal Statistical Society Series B: Statistical
  Methodology}
\bvolume{58}
\bpages{155--176}.
\end{barticle}
\endbibitem

\bibitem[\protect\citeauthoryear{Hoi, Liu and Chang}{2010}]{hoi2010semi}
\begin{barticle}[author]
\bauthor{\bsnm{Hoi},~\bfnm{Steven~CH}\binits{S.~C.}},
  \bauthor{\bsnm{Liu},~\bfnm{Wei}\binits{W.}} \AND
  \bauthor{\bsnm{Chang},~\bfnm{Shih-Fu}\binits{S.-F.}}
(\byear{2010}).
\btitle{Semi-supervised distance metric learning for collaborative image
  retrieval and clustering}.
\bjournal{ACM Transactions on Multimedia Computing, Communications, and
  Applications (TOMM)}
\bvolume{6}
\bpages{1--26}.
\end{barticle}
\endbibitem

\bibitem[\protect\citeauthoryear{Hoi et~al.}{2006}]{DCA}
\begin{binproceedings}[author]
\bauthor{\bsnm{Hoi},~\bfnm{Steven~CH}\binits{S.~C.}},
  \bauthor{\bsnm{Liu},~\bfnm{Wei}\binits{W.}},
  \bauthor{\bsnm{Lyu},~\bfnm{Michael~R}\binits{M.~R.}} \AND
  \bauthor{\bsnm{Ma},~\bfnm{Wei-Ying}\binits{W.-Y.}}
(\byear{2006}).
\btitle{Learning distance metrics with contextual constraints for image
  retrieval}.
In \bbooktitle{2006 IEEE Computer Society Conference on Computer Vision and
  Pattern Recognition (CVPR'06)}
\bvolume{2}
\bpages{2072--2078}.
\bpublisher{IEEE}.
\end{binproceedings}
\endbibitem

\bibitem[\protect\citeauthoryear{M{\"u}ller}{2007}]{muller2007information}
\begin{bbook}[author]
\bauthor{\bsnm{M{\"u}ller},~\bfnm{Meinard}\binits{M.}}
(\byear{2007}).
\btitle{Information retrieval for music and motion}.
\bpublisher{Springer}.
\end{bbook}
\endbibitem

\bibitem[\protect\citeauthoryear{Schmee}{1986}]{normal}
\begin{bmisc}[author]
\bauthor{\bsnm{Schmee},~\bfnm{Josef}\binits{J.}}
(\byear{1986}).
\btitle{An introduction to multivariate statistical analysis}.
\end{bmisc}
\endbibitem

\bibitem[\protect\citeauthoryear{Si et~al.}{2006}]{si2006collaborative}
\begin{barticle}[author]
\bauthor{\bsnm{Si},~\bfnm{Luo}\binits{L.}},
  \bauthor{\bsnm{Jin},~\bfnm{Rong}\binits{R.}},
  \bauthor{\bsnm{Hoi},~\bfnm{Steven~CH}\binits{S.~C.}} \AND
  \bauthor{\bsnm{Lyu},~\bfnm{Michael~R}\binits{M.~R.}}
(\byear{2006}).
\btitle{Collaborative image retrieval via regularized metric learning}.
\bjournal{Multimedia Systems}
\bvolume{12}
\bpages{34--44}.
\end{barticle}
\endbibitem

\bibitem[\protect\citeauthoryear{Wainwright}{2019}]{wainwright2019high}
\begin{bbook}[author]
\bauthor{\bsnm{Wainwright},~\bfnm{Martin~J}\binits{M.~J.}}
(\byear{2019}).
\btitle{High-dimensional statistics: A non-asymptotic viewpoint}
\bvolume{48}.
\bpublisher{Cambridge university press}.
\end{bbook}
\endbibitem

\bibitem[\protect\citeauthoryear{Wang and Gasser}{1997}]{wang1997alignment}
\begin{barticle}[author]
\bauthor{\bsnm{Wang},~\bfnm{Kongming}\binits{K.}} \AND
  \bauthor{\bsnm{Gasser},~\bfnm{Theo}\binits{T.}}
(\byear{1997}).
\btitle{Alignment of curves by dynamic time warping}.
\bjournal{The annals of Statistics}
\bvolume{25}
\bpages{1251--1276}.
\end{barticle}
\endbibitem

\bibitem[\protect\citeauthoryear{Xing et~al.}{2002}]{xing2002distance}
\begin{barticle}[author]
\bauthor{\bsnm{Xing},~\bfnm{Eric}\binits{E.}},
  \bauthor{\bsnm{Jordan},~\bfnm{Michael}\binits{M.}},
  \bauthor{\bsnm{Russell},~\bfnm{Stuart~J}\binits{S.~J.}} \AND
  \bauthor{\bsnm{Ng},~\bfnm{Andrew}\binits{A.}}
(\byear{2002}).
\btitle{Distance metric learning with application to clustering with
  side-information}.
\bjournal{Advances in neural information processing systems}
\bvolume{15}.
\end{barticle}
\endbibitem

\end{thebibliography}

\end{document}